\documentclass[11pt]{article}

\usepackage{booktabs} 
\usepackage[ruled]{algorithm2e} 

\SetAlFnt{\small}
\SetAlCapFnt{\small}
\SetAlCapNameFnt{\small}
\SetAlCapHSkip{0pt}
\IncMargin{-\parindent}

\usepackage{amsmath}
\usepackage{amsthm}
\usepackage{amsfonts}
\usepackage{mathtools} 
\usepackage{bbm} 
\usepackage{bm} 
\usepackage{graphicx}
\usepackage{color}
\usepackage{thmtools} 
\usepackage{thm-restate}

\DeclareMathOperator*{\argmin}{arg\,min}
\DeclareMathOperator*{\argmax}{arg\,max}

\numberwithin{theorem}{section}
\newtheorem{lemma}{Lemma}
\numberwithin{lemma}{section}
\newtheorem{corollary}{Corollary}
\numberwithin{corollary}{section}
\newtheorem{assumption}{Assumption}

\newcommand{\sycomment}[1]{}
\newcommand{\synew}[1]{#1}
\newcommand{\removedsy}[1]{}
\newcommand{\newshipra}[1]{#1}
\newcommand{\removedshipra}[1]{}
\newcommand{\sacomment}[1]{#1}

\newcommand{\1}{\mathbbm{1}}

\def\ddefloop#1{\ifx\ddefloop#1\else\ddef{#1}\expandafter\ddefloop\fi}
\def\ddef#1{\expandafter\def\csname bb#1\endcsname{\ensuremath{\mathbb{#1}}}}
\ddefloop ABCDEFGHIJKLMNOPQRSTUVWXYZ\ddefloop

\def\ddef#1{\expandafter\def\csname c#1\endcsname{\ensuremath{\mathcal{#1}}}}
\ddefloop ABCDEFGHIJKLMNOPQRSTUVWXYZ\ddefloop

\def\ddef#1{\expandafter\def\csname v#1\endcsname{\ensuremath{\boldsymbol{#1}}}}
\ddefloop ABCDEFGHIJKLMNOPQRSTUVWXYZabcdefghijklmnopqrstuvwxyz\ddefloop

\def\ddef#1{\expandafter\def\csname v#1\endcsname{\ensuremath{\boldsymbol{\csname #1\endcsname}}}}
\ddefloop {alpha}{beta}{gamma}{delta}{epsilon}{varepsilon}{zeta}{eta}{theta}{vartheta}{iota}{kappa}{lambda}{mu}{nu}{xi}{pi}{varpi}{rho}{varrho}{sigma}{varsigma}{tau}{upsilon}{phi}{varphi}{chi}{psi}{omega}{Gamma}{Delta}{Theta}{Lambda}{Xi}{Pi}{Sigma}{varSigma}{Upsilon}{Phi}{Psi}{Omega}{ell}\ddefloop

\usepackage[backend=bibtex,style=numeric,natbib=true]{biblatex}
\newcommand{\pcurvei}{\lcb{p_i}}
\newcommand{\preg}{\text{Pseudo-Regret}}
\newcommand{\gm}{\gamma m} 
\newcommand{\gim}{\gamma_i m}
\newcommand{\pualgo}{\log(e+\phi T)}
\newcommand{\pu}{\log(e+(\alpha+\beta) T)}

\newcommand{\detV}{V^{\text{det}}}
\newcommand{\detQ}{Q^{\text{det}}}
\newcommand{\detG}{G^{\text{det}}}
\newcommand{\stochV}{V^{\text{stoch}}}
\newcommand{\pcurve}{p^*}
\newcommand{\pe}{p_0} 
\newcommand{\lcb}[1]{\underline{#1}}
\newcommand{\reg}{\text{Regret}}
\newcommand{\detreg}{\text{Regret}^{\text{det}}}
\newcommand{\rev}{\text{Rev}}

\newcommand{\Xopt}{X^*}
\newcommand{\Xdet}{X^{\text{det}}}
\newcommand{\KL}{\text{KL}}
\newcommand{\alphaL}{L_\alpha}
\newcommand{\betaL}[1]{L_{\beta #1}}
\newcommand{\popt}{\pi^*}

\newcommand{\epochboundorder}{m^{2/3}\log\left(\frac{T}{\delta}\right) \left(\frac{1}{\alpha}+\frac{1}{\beta}+\frac{\beta}{\alpha^2}\right)\phi}

\newcommand{\regretboundorder}{m^{2/3} \log(m) \log(\frac{T}{\delta}) \left(\frac{1}{\alpha}+\frac{1}{\beta}+\frac{\beta}{\alpha^2}\right)\phi}

\author{ {\sf Shipra  Agrawal}\thanks{Industrial Engineering and Operations Research,
e-mail: {\tt sa3305@columbia.edu}} \\Columbia University
\and {\sf Steven Yin}\thanks{Industrial Engineering and Operations Research,
e-mail: {\tt sy2737@columbia.edu}} \\Columbia University
\and {\sf Assaf  Zeevi}\thanks{Graduate School of Business,
e-mail: {\tt assaf@gsb.columbia.edu}} \\Columbia University}

\title{Dynamic Pricing and Learning under the Bass Model}

\begin{document}
\date{}
\maketitle
\begin{abstract}
    We consider a novel formulation of the dynamic pricing and demand
    learning problem, where the evolution of demand in response to posted
    prices is governed by a stochastic variant of the popular Bass model with
    parameters $(\alpha, \beta)$ that are linked to the so-called
    ``innovation'' and ``imitation'' effects. Unlike the more commonly used
    i.i.d. and contextual demand models, in this model the posted price not
    only affects the demand and the revenue in the current round but also the future 
    evolution of demand, and hence the fraction of potential market size $m$  that can
    be ultimately captured. Finding a revenue-maximizing dynamic
    pricing policy in this model is non-trivial even in the full information
    case, where model parameters are known. In this paper, we consider the  more
    challenging {\it incomplete information} problem where dynamic pricing is applied in conjunction with
    learning the unknown model parameters, with the objective of optimizing the
    cumulative revenues over a given selling horizon of length $T$. Equivalently, the goal is to minimize the regret which measures the revenue loss of the algorithm relative to the {\it optimal} expected
    revenue achievable under the stochastic Bass model with market size $m$ and time horizon $T$. 
    
    Our main contribution
    is the development of an algorithm that satisfies a high probability regret guarantee of order $\tilde
    O(m^{2/3})$; where the market size $m$ is known a priori.  Moreover,  we show that no algorithm can incur smaller order of loss by deriving a matching lower bound.
    Unlike most regret analysis results, in the present problem the market size $m$ is the fundamental driver of the complexity; our lower bound in fact, indicates  that for any fixed $\alpha, \beta$, most non-trivial  instances of the problem have  constant $T$ and large $m$. We believe that this insight sets the problem of dynamic pricing under the Bass model apart from the typical i.i.d. setting and multi-armed bandit based models for dynamic pricing,  which typically focus only on the asymptotics with respect to time horizon $T$.
\end{abstract}

\newpage
\section{Introduction}
\label{sec: introduction}
\paragraph{\underline{Background and motivation}} The dynamic pricing and learning
literature, often referred to as ``learning and earning,'' has at its focal
point the objective of maximizing revenues jointly with inferring the
structure of a demand model that is a priori not known to the decision maker.
It is an extremely active area of current research that can essentially be
traced back to two strands of work. Within the computer science community,
the first paper on the topic is \cite{LeightonKleinberg03} that studies a
posted price auction with infinite inventory in which the seller does not
know the willingness-to-pay (or valuation) of buyers and must learn it over
repeated interactions. The problem is stateless, and demand is independent
from period to period; as such, it is amenable to a multi-armed bandit-type
approach and in fact can be solved near-optimally (in the sense of regret
bounds) using an ``explore-first-exploit-later'' pricing and learning
strategy. Various refinements and improvements have been proposed since in
what has become a very active field of study in economics, computer science
and operations research. The second strand of work originates in the
operations research community \cite{BesbesZeevi09} which focuses on the same
finite horizon regret criteria in the posted-price auction problem but with
limited inventory. This problem is sometimes referenced as ``bandits with
knapsacks'' due to the follow up work of
\cite{BadanidiyuruKleinbergSlivkins13} and subsequent papers. In that problem, the learning
objective is more subtle as the system ``state'' is changing over time (in
the dynamic programming full information solution, this state is given by the
level of remaining inventory and time). For further references and historical
notes on the development of the topic see, e.g., the recent survey
\cite{dynamic_pricing_learning_review}.

Most of the literature that has evolved from the inception points identified
above has focused on a relatively simple setting where given current pricing
decision, demand is independent of past actions and demand values. In
addition, much of the literature has focused on the ``stateless'' problem
setting, which provides further simplification and tractability. With the
evolution of online platforms and marketplaces, the focus on such homogeneous
modeling environments is becoming increasingly less realistic. For example,
platforms now rely more and more on online reviews and ratings to inform and
guide consumers. Product quality information is also increasingly available
on online blogs, discussion forums, and social networks,  that create further
word-of-mouth effects. One clear implication on the dynamic pricing
and learning problem is that the demand environment can no longer  assumed to be static,
meaning the demand model may change from one time instance to the next; for example, in
the context of online reviews, sales of the product trigger reviews/ratings,
and these in turn influence subsequent demand behavior etc.

While it is possible to model demand shifts and non-stationarities in a
reasonably flexible (nonparametric) manner within the dynamic pricing and
learning problem (see, e.g., \cite{keskin_zeevi_2017}, and
\cite{BesbesGurZeevi19} for a general MAB formulation), this
approach can be too broad and unstructured to be effective in practical dynamic
pricing settings. To that end, product diffusion models, such as the
popular Bass model \cite{Bass69, Bass04}, are known to be extremely robust and
parsimonious, capturing aforementioned word-of-mouth and imitation
effects on the growth in sales of a new product. 
The Bass diffusion model, originally proposed by Frank Bass in 1969
\cite{Bass69} has been extremely influential in marketing and management
science, often described as one of the most celebrated empirical generalizations
in marketing. It describes the process by which new products get adopted as an
interaction between existing users and potential new users. The Bass model
is attractive insofar as it creates a state-dependent evolution of market
response which is well aligned with the impact of recent technological
developments, such as online review platforms,  on the customer purchase
behavior. To that end, several recent empirical studies in marketing science and
econometrics utilize abundant social data from online platforms to quantify
the impact of word-of-mouth effect on consumer purchase behaviors and a new product
diffusion process (e.g., \cite{SENECAL2004, Chevalier2006, DELLAROCAS2007,
EAST2008, Iyengar2011, Toubia2014}, also see \cite{Dellarocas2003,
ToubiaSurvey2015} for literature surveys).

Motivated by these observations, the objective of this paper is to
investigate a novel formulation of the the dynamic pricing and demand
learning problem, where the evolution of demand is governed by the Bass
diffusion model, and where the parameters of this model are unknown a priori
and need to be learned from repeated interactions with the market.

\paragraph{\underline{The proposed model and main problem studied}}
The Bass diffusion model \cite{Bass69, Bass04} has two
parameters: the ``coefficient of innovation'' representing external influence
or advertising effect; and the ``coefficient of imitation'' representing
internal influence or word-of-mouth effect. The principle on which we
incorporate this model is that ``the probability of adopting by those who
have not yet adopted is a linear function of those who had previously
adopted.'' More precisely, the likelihood that a potential buyer buys a new
item at time $t$, given that he or she has not yet bought it, is a linear
function of the proportion of buyers at time $t$:
\begin{equation}
\label{eq:Bass1} 
\frac{f(t)}{1-F(t)} = \alpha+\beta F(t)
\end{equation}
where $F(t)$ and $f(t)$ are the distribution and density functions of first-purchase times, respectively. Here, the parameter $\alpha$ is the above referenced ``coefficient of innovation'', and $\beta$ is the ``coefficient of imitation.'' Let $m$ be the number of potential buyers,  i.e., the market size, and let $X_t$ be fraction of customers who has already adopted the product until time $t$. Then, $m X_t= m F(t)$ represents the cumulative sales (aka adoptions) up until time $t$,  and $m(1-X_t)$ is the size of remaining market yet to be captured. The instantaneous sales at time $t$, $m\frac{dX_t}{dt} = m f(t)$ can then be expressed as 
\begin{equation}
\label{eq:Bassnoprice} 
m \frac{dX_t}{dt} = \underbrace{m \alpha (1-X_t)}_{\text{sales due to external influence}} + \underbrace{\textstyle m \beta X_t (1-X_t)}_{\text{sales due to internal influence or imitation}}
\end{equation}
 
A generalization of the Bass diffusion model that can be harnessed for the
dynamic pricing context was proposed by Robinson and Lakhani
\cite{RobinsonLakhani1975}. In the latter model, $p_t$ denotes the price
posted at time $t$. Then, given previous adoption level $X_t$ the number of new adoptions at time instant $t$ is given by
\begin{equation}
\label{eq:detrministicBassintro}
m\frac{d X_t}{d t} = m e^{-p_t} (\alpha+\beta X_t) (1-X_t).
\end{equation}
Thus, the current price affects not only the immediate new adoptions and revenue, but also future adoptions due to its  dependence on the adoption level $X_t$, which essentially forms the state at time $t$ in this stateful demand model.
Finding a revenue-maximizing optimal pricing trajectory in
such a dynamic demand model is non-trivial even when the model parameters are
known, e.g., see \cite{experience_curve} for some characterizations. 

In this
paper, we consider a stochastic variant of the above Bass model, where 
customers arrive sequentially and the number of customer arrivals (aka adoptions) 
$d_t$ until time $t$ is  a non-homogeneous Poisson process with rate $\lambda_t$ given by the right hand side of \eqref{eq:detrministicBassintro}.
More details on the stochastic model are provided in Section~\ref{sec: problem
formulation} where we describe the full  problem formulation and performance objectives.  

The problem of dynamic pricing under demand learning can then be described, informally, as follows: the learner (seller) is required to dynamically
choose prices $\{p_t\}$ to be posted at  time $t\in [0,T]$, where $p_t$ is chosen based on the past
observations that include the number of customers arrived so far, their times of arrival,
and the price decisions 
made in the past. The number of customers $d_t$ arriving until any time $t$ is given by the stochastic Bass model. Note that we use the term ``customer arrival'' and ``customer adoption'' interchangeably to mean the same thing, i.e., every customer arriving at time $t$ adopts the product and pays the price $p_t$.
We assume that the
size of the market $m$ and the horizon $T$ are known to the
learner, but the Bass model parameters  $\alpha, \beta$ (i.e., coefficient of innovation and
coefficient of imitation) are unknown. The aim is to maximize the 
cumulative revenue over the horizon  $T$, i.e.,$\sum_{d=1}^{d_T} p_d$
via a suitable sequence of prices, where $p_d$ is the price paid by the $d^{th}$ customer and $d_T$ is the total number of adopters until time $T$. In essence, we will not be directly maximizing this quantity but rather, and
much in line with the online learning and multi-armed bandits literature,
will focus on evaluating a suitable notion of regret and establishing
``good'' regret performance of our proposed pricing and learning algorithm.



\paragraph{\underline{Main contributions}} The paper makes two significant advances in the study of the aforementioned Bass model learning and earning problem.  First, 
we present a learning and pricing
algorithm that achieves  a high probability $\tilde O\left(m^{2/3}\right)$ regret
bound. Second, under certain mild restrictions on algorithm design, we provide a matching  lower bound, showing that any algorithm must incur order $\Omega(m^{2/3})$ regret for this problem.   The precise statements of these results are provided as Theorem \ref{thm:main} and Theorem \ref{thm:mainlowerbound}, in Section \ref{sec: regret upper bound} and \ref{sec: regret lower bound} respectively. 
Hence the ``price'' of incomplete information and the ``cost'' of learning it on the fly, are reasonably small.  Unlike most regret analysis results, in the present setting the market size $m$ is the fundamental driver of the problem complexity; our lower bound, in fact,  indicates that for any fixed $\alpha, \beta$, most non-trivial  instances of the problem have  constant $T$ and large $m$. 
This is also reflected in the regret of our algorithm which depends sublinearly on market size $m$, and only logarithmically on the time horizon $T$. This insight sets the problem of dynamic pricing under Bass model uniquely apart from the typical i.i.d. and multi-armed bandit-based models for dynamic pricing. 

\paragraph{\underline{Other Related Work}}
\label{sec: literature review}
Several other models of non-stationary demand have been considered in recent
literature on dynamic pricing and demand learning. \citet{denboer2015} studies
an additive demand model where the market condition is determined by the sum
of an unknown market process and an adjustment term that is a known function
of the price. The paper numerically studies several sliding window and weight decay
based algorithms, including one experiment where their pricing strategy
is tested on the Bass diffusion model. However, they do not provide any
regret guarantees under the Bass model. \citet{besbes_zeevi_2011} and
\citet{chen2019dynamic} consider the pricing problem under a piece-wise
stationary demand model. \citet{keskin_zeevi_2017} consider a two-parametric
linear demand model whose parameters are indexed by time, and the total
quadratic variation of the parameters are bounded. Many of the existing
dynamic pricing and learning algorithms borrow ideas from related settings in
multi-armed bandit literature. \citet{garivier2008ucb} and
\citet{cao2018nearly} study piecewise stationary bandits based on sliding
window and change point detection techniques.
\citet{BesbesGurZeevi19,russac2019weighted} and
\citet{contextualBanditsNonstationary} study regret based on an appropriate
measure of reward variation. In rested and restless bandits literature,
\citet{tekin_2012} and \citet{bertsimas2000} also study a related problem
where the reward distribution depends on the state which follows an unknown
underlying stochastic process. However, the above-mentioned models of non-stationarity
are typically very broad and unstructured, and fundamentally different from
the stateful structured Bass diffusion model studied here. 

There is also much recent work on learning and regret minimization in stateful models using general MDP and reinforcement learning frameworks (for example \cite{yang2020reinforcement,agrawal2017posterior,jaksch2010near}). However, these results typically rely on having settings where each state can be visited many times over the learning process. This is ensured either because of an episodic MDP setting (e.g., \cite{yang2020reinforcement}), or through an assumption of communicating or weakly communicating MDP with bounded diameter, i.e., a bound on the number of steps to visit any state from any starting state under the optimal policy (e.g., see \cite{jaksch2010near, agrawal2017posterior}). Our setting is non-episodic, and every state is transient -- the state is described by the number of cumulative adopters so far and the remaining time, both of which can only move in one direction.
Therefore, the results in the above and related papers on learning general MDPs are not applicable to our problem setting. 

In the marketing literature, there is significant work on using the Bass
model to forecast demand before launching a product. Stochastic variants of
Bass model have also been considered previously, e.g., in \citet{Grasman2019,
niu2002}. In addition to the work mentioned in the introduction,
\citet{lee2014,FAN201790,bassAndDeepLearning2020, Grasman2019} present
empirical methods for using historical data from similar products or from
past years to estimate Bass model parameters, in order to obtain reliable
forecasts. In this context, we believe our proposed dynamic pricing and
learning algorithm may provide an efficient method for adaptively estimating
the Bass model parameters while optimizing revenue. However, the focus of
this paper is on providing theoretical performance guarantees; the empirical
performance of the proposed method has not been studied.

The work most closely related to ours is a parallel recent paper by \citet{zhang2020data}.
In \cite{zhang2020data}, the authors consider a dynamic pricing problem under a similar stochastic Bass model setting as the one studied here. They propose an algorithm based on Maximum Likelihood Estimation (MLE) that is claimed to achieve a logarithmic regret bound  of $\tilde O(\log(mT))$. At first glance this regret bound seems to contradict our lower bound\footnote{Technically our lower bound only holds for algorithms that satisfy certain conditions (Assumption \ref{assum: const price in between arrivals} and \ref{assum:price upper bound}) stated in Section \ref{sec: regret lower bound}. However, we believe it is still applicable to the algorithm in \cite{zhang2020data}. Their algorithm with its  constant upper bound on prices seems to satisfy both of our assumptions.} of $\Omega(m^{2/3})$. However, it appears  the results in \cite{zhang2020data} may have some mistakes (see, in particular,  the current statement and proofs of Lemma 3, Theorem 5 \synew{and Theorem 6} in \cite{zhang2020data} which we believe contain the aforementioned inconsistencies). \sycomment{I removed the appendix part. I think the wording is fine.} \removedsy{For completeness, interested readers can find some details on these potential inconsistencies present in \cite{zhang2020data} in Appendix \ref{sec:michiganpaper}. }
To the best of our understanding, these inconsistencies cannot be fixed without changing the current results in a significant way.
\section{Problem formulation and preliminaries}
\label{sec: problem formulation}
\paragraph{\underline{The stochastic model and problem primitives}}There is an unlimited supply of durable goods of a single type to be sold in a
market with $m$ customers. We consider a dynamic pricing problem with sequential customer arrivals over a time interval $[0,T]$. 
We denote by $d_t$, the number of arrivals by time $t$; with $d_0=0$.
At any given time $t$, the seller  observes $d_t$, the number of arrivals so far as well as their arrival times $\{\tau_1, \tau_2, \dots, \tau_{d_t}\}$, and chooses a new price $p_{t}$ to be posted for times $t'>t$. The seller can use the past information to update the price any number of times until the end of time horizon $T$. As mentioned earlier, in our formulation a customer arriving at time $t$ immediately adopts the product and pays the posted price $p_t$, and therefore the terms ``adoption" and ``arrival" are used interchangeably throughout the text.


The customer arrival process for any pricing policy is given by a stochastic Bass diffusion model with (unknown) parameters $\alpha, \beta$. In this model, number of arrivals $d_t$ by time $t$ is a non-homogeneous Poisson point process \footnote{
A counting process $\{d_t,t\ge 0\}$ is called a non-homogeneous Poisson process with rate $\lambda_t$ if all the following conditions hold: (a) $d_0=0$; (b) $d_t$ has independent increments; (c) for any $t\ge 0$ we have $\Pr(d_{t+\delta} - d_t =0) = 1-\lambda_t \delta + o(\delta)$, $\Pr(d_{t+\delta} - d_t =1) = \lambda_t \delta + o(\delta)$, $\Pr(d_{t+\delta} - d_t \ge 2) = o(\delta)$. 
} with rate $\lambda_t$ given by  \eqref{eq:detrministicBassintro}, the adoption rate in deterministic Bass diffusion model  \cite{RobinsonLakhani1975}. That is, 
$$\lambda_t = me^{-p_t}(\alpha+\beta X_t)(1-X_t)$$
where $X_t=d_t/m$. 
For convenience of notation, we define  function 
$$\lambda(p, x) := me^{-p}(\alpha+\beta x)(1-x),$$ so that $\lambda_t= \lambda(p_t, X_t)$, with $X_t=d_t/m$.


The seller's total revenue is simply the sum of the prices paid by the customers who arrived before time $T$, under the dynamic pricing algorithm used by the seller. We denote by $p_d$, the price paid by $d^{th}$ customer, i.e., $p_d= p_{\tau_d}$. Then, the revenue over  $[0,T]$ is given by:
\begin{align*}
\rev(T) = \sum_{d=1}^{d_T} p_d.
\end{align*}
The optimal dynamic pricing policy is defined as the one that  maximizes the total expected revenue $E[\rev(T)]$. We denote by $\stochV(T)$, the total expected revenue under the optimal dynamic pricing policy. Then, regret is defined as the difference between the optimal expected revenue $\stochV(T)$ and seller's revenue, i.e., 
\begin{equation}
\label{eq: regret}
\reg(T) = \stochV(T) - \rev(T).    
\end{equation}

In this paper, we aim to provide a dynamic learning and pricing algorithm with a high probability upper bound of $\tilde O(m^{2/3})$ on regret, as well as a closely matching lower bound. 
\sycomment{I changed this part a bit:}
\removedsy{However, the optimal pricing policy and optimal revenue $\stochV(T)$ for the stochastic Bass model described
above is difficult to derive and analyze even when the model parameters are known, which makes it difficult to characterize the regret of our learning problem. 
To aid algorithm design and regret analysis, we define the notion of ``pseudo-regret'', which measures regret against a more tractable benchmark of $\detV(T)$, the {\it optimal revenue under the deterministic Bass model}. This follows a widely used approach in analysis of stochastic systems which tackles the deterministic ``skeleton" or fluid model as a stepping stone to developing insights for policy design and complexity drivers.   Later, we show that pseudo-regret upper bounds the regret measure $\reg(T)$ defined earlier, and further is within $\tilde O(\sqrt{m})$ of $\reg(T)$. Therefore, it will suffice to analyze the pseudo-regret instead of regret for both upper and lower bound purposes.}
\synew{Instead of analyzing the regret directly, we define a notion of ``pseudo-regret,'' which measures regret against the {\it optimal revenue under the deterministic Bass model} ($\detV(T)$). This is useful because as we will show later in \eqref{eq: opt_p_model_in_opt_trajectory}, there is a simple expression for the optimal prices in the deterministic Bass model. This can be leveraged using a fluid model approach, widely used in the analysis of stochastic systems,  which targets a deterministic ``skeleton"  as a stepping stone toward developing insights for policy design and complexity drivers.   Later, we show more formally that the pseudo-regret upper bounds the regret measure $\reg(T)$, defined earlier, and further is within $\tilde O(\sqrt{m})$ of $\reg(T)$. This justifies our focus on the pseudo-regret for both the purpose of developing lower bounds, as well as analysis and upper bounds on policy performance.}

Next, we discuss the expressions for optimal pricing policy and optimal revenue under the deterministic Bass model, and use those to define the notion of pseudo-regret.



\paragraph{\underline{Optimal revenue and pricing policy under the deterministic Bass model}}
Recall (refer to introduction) that the adoption process under the deterministic Bass model is described as follows:
given the current adoption level $X_t$ and
 price $p_t$ at time $t$, the new adoptions are generated deterministically with rate \footnote{A subtle but important difference between this deterministic model vs. the stochastic Bass model  introduced earlier is that in the deterministic model, the increment $m dX_t$ in adoptions is fractional. On the other hand, in our stochastic model, the number of customer arrivals (or adoptions) $d_t$  is a counting process with discrete increments. This difference will be  taken into account later when we compare our pseudo-regret to the original definition of regret. }
\begin{equation}
    \label{eq:detrministicBass}
    m\frac{d X_t}{d t} = m e^{-p_t} (\alpha+\beta X_t) (1-X_t).
\end{equation}


The optimal price curve for the deterministic Bass model is then defined as the price trajectory $\{p_t\}$ that maximizes the total  revenue $m\int_{0}^T p_t dX_t$ under the above adoption process.  We denote by $\detV(T)$ the total  revenue under the optimal price curve.

An analytic expression for the optimal price curve under the deterministic Bass model is in fact known. It can be 
derived using optimal control theory  (see Equation (8) of \cite{experience_curve}). For a Bass model with parameters $\alpha, \beta$, horizon $T$, and initial adoption level $0$, the price at adoption level $x$ in the optimal price curve  is given by the following expression:

\begin{equation}
\pcurve(x, \alpha, \beta) 
= 1+\log\left(\frac{(\alpha+\beta x)(1-x)}{(\alpha+\beta \Xopt_T)(1-\Xopt_T)}\right),
\label{eq: opt_p_model_in_opt_trajectory}
\end{equation}
where $\Xopt_T$, the adoption level at the end of horizon $T$, is given by the following equations:
\begin{align}
        \Xopt_T = \frac{1}{e}(\alpha+\beta \Xopt_T)(1-\Xopt_T)T, \label{eq: intro XT_model_quadratic_sol1}\\
        \text{ or, more explicitly} \nonumber \\
        \Xopt_T=\frac{T(\beta-\alpha)-e+\sqrt{[T(\beta-\alpha)-e]^2+4\alpha\beta T^2 }}{2\beta T}.\label{eq: intro XT_model_quadratic_sol2}
\end{align}
For completeness, a derivation of the above expression of $\Xopt_T$ is included in Appendix~\ref{sec: det opt price properties}.
Using the above notation for optimal price curve, we can write $\detV(T)$, the optimal total revenue, as:
$$\detV(T) =  m\int_{0}^{\Xopt_T}\pcurve(x, \alpha, \beta) \,dx.$$

\paragraph{\underline{Pseudo-Regret}}
We define pseudo-regret as the difference between $\detV(T)$, the optimal total revenue in the deterministic Bass model, and the algorithm's total revenue $\rev(T)$. That is, 
\begin{equation}
    \label{eq:pseudo-reg}
\preg(T)=\detV(T) - \rev(T).
\end{equation}
Essentially, pseduo-regret replaces the benchmark of  stochastic optimal revenue $\stochV(T)$ used in the regret definition by deterministic optimal revenue $\detV(T)$. We show that in fact the deterministic optimal revenue is a stronger benchmark, in the sense that it is always larger than the stochastic optimal revenue. Furthermore, we show that it  is within $\tilde{O}(\sqrt{m})$ of the stochastic optimal revenue. To prove this relation between the two benchmarks we demonstrate a concavity property of deterministic optimal revenue which is crucial for our results. Specifically, we define an expanded notation $\detV(x,T)$ as the deterministic optimal revenue starting from adoption level $x$ and remaining time $T$. Note that $\detV(T) = \detV(0, T)$. Then, we show that $\detV(x,T)$ is concave in $x$ for any $T$, and any market parameters $m,\alpha, \beta$. More precisely, we prove the following key lemma.
\begin{restatable}[Concavity of deterministic optimal revenue]{lemma}{concavity}
    For any deterministic Bass model,  $\detV(x, T)$, defined as the optimal revenue  starting from adoption level $x$ and remaining time $T$,  is concave in $x$, 
    for all $T\geq 0$, and all adoption levels $x\in[0,1]$.
    \label{lem: concavity of deterministic value function}
\end{restatable}

Using these observations, we can prove the following relation between $\preg(T)$ and $\reg(T)$; all proofs can be found in the appendix.
\begin{lemma}[Pseduo-regret is close to regret]
\label{lem: pseudo regret and regret}
For any $T\ge 0$, 
$$\preg(T) \ge \reg(T)$$
\removedsy{And, for markets with $ T,m, \alpha, \beta$ such that   $m \ge \lfloor m\Xopt_T\rfloor + 2$, we have}
$$\preg(T)\le \reg(T)+ O\left(\sqrt{m}\log^2(m) + \log((\alpha+\beta)T)\log^2(m\log((\alpha+\beta)T)) \right).$$ 
\end{lemma}
\noindent A simpler summary of this result can be stated as 
$$\reg(T) \le \preg(T) \le \reg(T)+\tilde O(\sqrt{m}).$$
 Therefore, an upper bound on $\preg(T)$ implies the same upper bound on $\reg(T)$. And, a lower bound on $\preg(T)$ implies a lower bound on $\reg(T)$ within $\tilde O(\sqrt{m})$. In the rest of the paper, we therefore derive bounds on pseudo-regret only.



\paragraph{\underline{Notation conventions}} Throughout the paper, if a potentially fractional number (like $m \Xopt_T$, $m^{2/3}$, $\gamma m$ etc.) is used as a whole number (for example as number of customers or as boundary of a discrete sum) without a floor or ceiling operation, the reader should assume that it is rounded {\it down} to its nearest integer.


\section{Algorithm Description}
\label{sec: proposed algorithm}

The concavity property of deterministic optimal revenue and the implied relation between pseudo-regret and regret derived in Lemma \ref{lem: pseudo regret and regret} suggests that deterministic optimal revenue provides a benchmark that is comparable to the stochastic optimal revenue. Further, this benchmark is more tractable than the stochastic optimal due to the known\synew{ and relatively simple} analytical expressions for optimal pricing policy, as stated in Section \ref{sec: problem formulation}. Using these insights, our algorithm is designed to essentially follow  (an estimate of) the optimal price curve for the deterministic model  at every time. 
We believe this approach could be applied to other finite
horizon MDP problems where such concavity property holds, which may be of independent interest. We now describe our algorithm.
\paragraph{\underline{Algorithm outline}} 
Our algorithm is provided, as input, the market size $m$, 
and a constant upper bound $\phi$ on $\alpha+\beta$. For many applications, it is reasonable to assume that $\alpha +\beta \le 1$, i.e., $\phi=1$, but we allow for more general settings. The market parameters $\alpha, \beta$ are unknown to the algorithm.

The algorithm alternates between using a low ``exploratory'' price aimed at
observing demand in order to improve the estimates of model parameters, and
using (a lower confidence bound on) the deterministic optimal prices for the estimated market parameters. Setting the exploratory price $\pe$ as $0$ suffices for our analysis, but more generally it can be set as any 
lower bound on the deterministic optimal prices, i.e., we just need $0\le \pe\leq \pcurve(x,\alpha, \beta), \ \forall x$. 
Using a non-zero exploratory price could be better in practice.

\removedshipra{Moreover, our algorithm keeps the price constant in between arrivals, i.e., we set the price for customer $d$ when customer $d-1$ arrives (at time $\tau_{d-1}$), and keep that price fixed until $d$ arrives. This might seem like a handicap on our algorithm, since the optimal pricing policy could potentially change the offered price in between arrivals. However, we are able to achieve good regret bound in spite of this restriction.}

Our algorithm changes price only on arrival of a new customer, and holds the price constant in between two arrivals. The prices are set as follows. The algorithm starts with using an exploratory price $\pe$ for the first
$\gm$ customers, where $\gamma=m^{-1/3}$.
The prices $p_1,\ldots, p_{\gm}$ and the observed arrival times $\tau_{1}, \ldots, \tau_{\gm}$ for the first $\gm$ customers 
are then used to obtain an estimation of $\alpha$, and a high probability 
error bound $A$ on this estimate. The estimate of parameter $\alpha$ is not
updated in subsequent time steps.
The algorithm then proceeds in epochs $i=1, 2, \ldots K$. Let $\gamma_i =
2^{i}\gamma$. Epoch $i$ starts right after customer $\gamma_im$ arrives and
ends either when customer $2\gamma_im$ arrives, or when we reach the end of
the planning horizon $T$.
In the beginning of each epoch $i$, the exploratory price $\pe$ is again
offered for the first $\gm$ customers in that epoch. The arrival times
observed from these customers are used to update the estimate of the $\beta$
parameter and its' high probability error bound $B_i$. For the remaining
customers in epoch $i$,
the algorithm offers a lower confidence bound $\pcurvei$ on the
deterministic optimal price $\pcurve\left(\frac{d-1}{m},\alpha, \beta\right)$
computed using the current estimates of $\alpha, \beta$ and their error bounds.

\begin{algorithm}[t]
    \SetAlgoLined
    Input: Horizon $T$, market size $m$, $\delta \in (0,1)$, a constant upper bound $\phi$ on $\alpha+\beta$.\\
    Initialize: $X_0=0$, $\gamma = m^{-1/3}$, $\gamma_{i}=2^{i-1}\gamma$ for $i=1, 2,\ldots$, and exploratory price $p_0=0$ \\
    Offer $\pe$ for the first $\gm$ customers. \\
    Use observed arrival times of the first $\gm$ customers to compute $\hat \alpha$ 
    according to \eqref{eq:alphahat}.
    Also, obtain a high
    probability estimation bound on $|\alpha-\hat \alpha|\leq A$ where $A$ is as defined in \eqref{eq:ABdef}.\\
    \For{$i\gets 1, 2, 3, ...$}{
        Post price $\pe$ for customers $d=\gim + 1,\ldots, \gim+\gm$.\\
        Use $\hat \alpha$ along with the observed arrival times of the first $\gm$ customers in the current epoch to 
        calculate a new estimate $\hat \beta_i$ according to \eqref{eq:betahat}. Also
        update the high probability bound $|\beta - \hat
        \beta_i| \leq B_i$, with $B_i$ as defined in \eqref{eq:ABdef}.\\ 
        Use $\hat \alpha, A, \hat \beta_i, B_i$ to compute price $\pcurvei(\frac{d-1}{m})$ according to \eqref{eq: lcb on opt price}. 
         Post price  $\pcurvei(\frac{d-1}{m})$
        for $d=\gim+\gm+1, \ldots, 2\gim$, or until end of time horizon $T$.\\
        If $\gamma_i\geq \frac{1}{3}$ then extend the current epoch until the
        end of the planning horizon.
    }
    \caption{Dynamic Learning and Pricing under Bass Model}
    \label{alg:main}
\end{algorithm}

\paragraph{\underline{Estimation and price computation details}} Since our algorithm fixes prices between two arrivals, the arrival rate of the (Poisson) adoption  process  is constant in between arrivals, which in turn implies that the inter-arrival times are exponential random variables. 
This greatly simplifies the estimation procedure for $\alpha, \beta$, which we describe next.

Estimates  $\hat \alpha$, and $\hat \beta_i$ for every epoch $i$ are calculated 
using the following equations which match the observed inter-arrival times to the expected value of the corresponding exponential random variables:
\begin{equation}
    \label{eq:alphahat}
    \frac{1}{e^{-\pe} \hat \alpha m} \lfloor\gamma m\rfloor = \tau_{\gm}
\end{equation}
\begin{equation}
    \label{eq:betahat}
    \frac{\lfloor\gamma m\rfloor}{e^{-\pe}(\hat \alpha +\hat \beta_i \gamma_i)(1-\gamma_i)m} = \tau_{\gim+\gm} - \tau_{\gim} .
\end{equation}
Also, we define
\begin{equation}
        \label{eq:ABdef}
        \textstyle
        A\coloneqq\frac{8\phi}{(1-\gamma)^2}\sqrt{8\log(\frac{2}{\delta})}\ m^{-1/3}, \ B_i \coloneqq
    \frac{16\phi}{\gamma_i(1/3-\gamma)^2}\sqrt{8\log(\frac{2}{\delta})}\ m^{-1/3}.
    \end{equation}

Then, using concentration results for exponential random variables, 
we can show that the error in estimates $\hat \alpha$ and $\hat \beta_i$ are bounded $A$ and $B_i$ respectively.   Specifically, we prove the following result. The proof is in Appendix~\ref{sec: upper bound proofs}. 

\begin{restatable}{lemma}{alphahatbetahatbound}
    \label{lem:alphahatbetahatbound}
    Assume $m^{1/3} \geq
    64\frac{(\alpha+\beta)^2}{\alpha^2}\sqrt{8\log(\frac{2}{\delta})}$. Then with probability $1-\delta$, $|\alpha-\hat \alpha|
    \leq A$.
    And for any $i=1,\ldots, K$, 
    with probability $1-\delta$, $|\beta-\hat \beta_i|\leq B_i$.
\end{restatable}
Note that only the estimate and error bounds ($\hat \beta_i$ and  $B_i$)  for the parameter $\beta$ are updated in every
epoch. The estimate and error bound ($\hat \alpha$ and $A$) for $\alpha$ is computed once in the
beginning and remains fixed throughout the algorithm. 
Given $\hat \alpha, A, \hat\beta_i B_i$, 
we define the price $\pcurvei(d/m)$ to be used by the algorithm in epoch $i$ as 
follows. For any $d\le m$,
\begin{equation}
    \label{eq: lcb on opt price}
    \begin{split}
    \pcurvei(\frac{d}{m}) &= \text{clamp}\left(\pcurve(\frac{d}{m}, \hat\alpha, \hat\beta_i) - \alphaL A - \betaL{i} B_i,\  [0, \pualgo] \right),\\ 
    \text{ where \ \ } & \alphaL = \frac{2}{\hat\alpha - A} + \frac{\hat\beta_i+B_i}{(\hat\alpha-A)^2},\\
    & \betaL{i}  = \frac{3}{\hat\alpha-A} + \frac{3}{\hat\beta_i-B_i}.
    \end{split}
\end{equation}
Here $\pcurve(\cdot, \cdot, \cdot)$ is the optimal deterministic price curve as given by \eqref{eq: opt_p_model_in_opt_trajectory}. In above, we clamped the price to the range $[0, \pualgo]$. That is, if the computed price is less than $0$ we set $\pcurvei$ to $0$; and if it is above  $\pualgo$, which is an upper bound on the optimal price (proven later in Lemma \ref{lem:priceupperbound}), we set $\pcurvei$ equal to this upper bound. 

Later in Lemma~\ref{lem: lcb on opt
price}, we show that the quantities $\alphaL, \betaL{i}$ play the role of Lipschitz constants for the optimal price curve:  
they provide high probability upper bounds on the derivatives of the optimal price curve
with respect to $\alpha, \beta$ respectively. 
Consequently (in Corollary~\ref{cor: lcb on opt price}), we can show that with high probability the price defined in \eqref{eq: lcb on opt price} will be lower than the corresponding deterministic optimal price. The reason 
that the algorithm uses  a lower confidence bound on the optimal price is that we want to acquire at least  as many customers in time $T$ as the optimal
trajectory. The intuition here is that losing customers (and the entire
revenue associated with those customers) can cost a lot more than losing a
little bit of revenue from each customer. 

Our algorithm is described in detail as Algorithm~\ref{alg:main}.


\section{Regret Upper Bound}
\label{sec: regret upper bound}

The main result from this section is the following  upper bound on the pseudo-regret of the algorithm proposed in the previous section. Since pseudo-regret upper bounds regret (refer to Lemma \ref{lem: pseudo regret and regret}) it directly implies the same upper bound on $\reg(T)$. 
\begin{restatable}[Regret upper bound]{theorem}{mainupperbound}
    \label{thm:main}
    For any
    market with parameters $\alpha, \beta, \alpha+\beta \le \phi$, market size $m$, and time horizon $T$,
    Algorithm \ref{alg:main} achieves the following regret bound
    with probability $1-\delta$,  
    $$\textstyle 
    \preg(T)\le O\left(\regretboundorder\right) = \tilde{O}(m^{3/2}).
    $$
\end{restatable}

\paragraph{\underline{Proof Intuition}} We first give an intuitive
explanation for why our algorithm works well. \sycomment{Also changed this part a bit:}\removedsy{As mentioned earlier in the introductory sections, one of the key challenges we
face in this problem is that we do not completely understand  the optimal pricing policy in the
stochastic Bass model even when the market parameters $\alpha, \beta$ are known. It is difficult to analyze regret when the benchmark is unknown. We get around this issue by replacing
the benchmark with a stronger one, but one that we have a better
understanding of. Specifically, the definition of pseudo-regret and Lemma~\ref{lem: pseudo regret and regret} 
allows us to replace the benchmark from the stochastic optimal revenue 
$\stochV(T)$ to deterministic optimal revenue $\detV(T)$.} 
\synew{As mentioned earlier in the introductory sections, there is a simple closed form expression for the optimal prices in the deterministic model for any $\alpha, \beta, T$ (see \eqref{eq: opt_p_model_in_opt_trajectory}--\eqref{eq: intro XT_model_quadratic_sol2}). Moreover, the definition of pseudo-regret and Lemma~\ref{lem: pseudo regret and regret} 
allows us to replace the stochastic optimal revenue 
$\stochV(T)$ with the deterministic optimal revenue $\detV(T)$. }
\removedsy{This is helpful because we have a \synew{simple} analytical
expression of the optimal price curve for any $\alpha, \beta, T$ in
the deterministic Bass model (see \eqref{eq: opt_p_model_in_opt_trajectory}--\eqref{eq: intro XT_model_quadratic_sol2}).} Our algorithmic strategy is then to follow (an
estimate of) the deterministic optimal price trajectory, and show that the resulting revenue is close to the deterministic optimal revenue with high probability. 

To prove this, we show that the prices $\pcurvei$ used by our algorithm were set so that they lower bound the deterministic optimal price  with high probability. Intuitively, using a lower price would ensure that the algorithm sees at least as many as optimal number of customer adoptions in horizon $T$, so that the gap in revenue can be bounded simply by the gap in prices paid by those customers. The
final piece of the puzzle is to show that we can learn or estimate $\alpha, \beta$ at a
fast enough rate so that the estimated prices are increasingly close to the optimal price and we do not lose too much revenue from learning.

\paragraph{\underline{Proof Outline}} In the remainder of this section we outline the proof of
Theorem~\ref{thm:main} in four steps. All the missing proofs from this section can be
found in Appendix~\ref{sec: upper bound proofs}.

\paragraph{{Step 1 (Bounding the estimation errors)}}  In Lemma \ref{lem:alphahatbetahatbound}
we provided a high probability upper bound on the estimation error in  $\alpha, \beta$ in each epoch of the algorithm.
 Using these error bounds and the definition of price $\pcurvei$ in \eqref{eq: lcb on opt price}, we can show that the prices offered by the algorithm are close to optimal and lower bound the corresponding optimal price with high probability. Specifically, we prove the following result.

\begin{restatable}[Error bounds for estimated prices]{lemma}{lcbonoptprice}
    \label{lem: lcb on opt price}
    Given any market parameters $\alpha, \beta$ and their estimations
    $\hat\alpha, \hat\beta_i$ that satisfy $|\alpha-\hat \alpha|\leq A$, $|\beta-
    \hat \beta_i|\leq B_i$, $\hat\alpha-A > 0$, $\hat\beta_i-B_i>0$, then for every $d=0,\ldots, m-1$,
    $$ \textstyle \left|\pcurve(\frac{d}{m},\hat\alpha, \hat\beta_i) -
    \pcurve(\frac{d}{m}, \alpha, \beta)\right| \leq \alphaL A + \betaL{i}
    B_i,$$
    where $\alphaL, \betaL{i}$ are as defined in \eqref{eq: lcb on opt price}.
\end{restatable}

From the expressions of $A$ (error bound for $\hat \alpha$) vs. $B_i$ (error bound for $\hat \beta_i$ in epoch $i$) in  \eqref{eq:ABdef}, 
observe that $A=\tilde{O}(m^{-1/3})$ and $B_i=\tilde{O}(\frac{1}{\gamma_i}m^{-1/3})$. Therefore, the estimation of $\beta$ is the bottleneck here: it has an extra
$\frac{1}{\gamma_i}$ factor in the error bound. This is because the imitation
factor $\beta$ is multiplied with the current adoption level in the
definition of the Bass model (see \eqref{eq:Bassnoprice}). This means that the
estimation error on $\beta$ is likely to be large when the adoption level is low.
This is why the algorithm needs to keep updating the estimate of $\beta$ in each epoch but not
$\alpha$. 

The above lemma and the definition of $\pcurvei$ immediately implies the following.
\begin{corollary}[Lower confidence bound on optimal price]
    \label{cor: lcb on opt price}
    Under the same assumptions as in Lemma~\ref{lem: lcb on opt price},  and given the definition of $\pcurvei$ in \eqref{eq: lcb on opt price}, we have that 
    for every $d=0,\ldots,m-1$,
    $\pcurvei(\frac{d}{m}) \leq \pcurve(\frac{d}{m}, \alpha, \beta)$.
\end{corollary}
\paragraph{Step 2 (Bounding the revenue loss due to lower price).}
    Using the fact that there are $\gamma_i m$ customers in epoch $i$, and the  price difference bound in Lemma~\ref{lem: lcb on opt price} 
    in Step 1, we can show that we lose at most 
    {$\tilde O\left(\gamma_i m \left(\alphaL m^{-1/3}
    +\betaL{i}\frac{1}{\gamma_i} m^{-1/3} \right)\right) = $ }
    $\tilde O\left(
    m^{2/3}\right)$ revenue per epoch. This loss bound per epoch is when compared to the optimal revenue earned
    from the same $\gamma_i m$ customers, assuming that the stochastic trajectory is given as much time as needed to reach the same number of customers. \removedshipra{Here $\alphaL, \betaL{i}$ are the Lipschitz constants
    for the optimal price function with respect to change in $\alpha, \beta$, as defined in
    \eqref{eq: lcb on opt price}. }
    More precisely, let $\rev_i$ denote the algorithm's revenue from the $(2\gim \wedge m\Xopt_T) - \gim$ customers in an epoch $i$ completed by the algorithm, and $\detV_i$ denote the potential revenue from those customers if they paid the deterministic optimal price instead. Then, we prove the following bound. 

\begin{restatable}{lemma}{epochbound}
    \label{lem:epochbound}
    For any epoch $i$ in the algorithm, with probability $1-\delta$,
    $$
    \textstyle
        \detV_i -\rev_i  \le  O\left(\epochboundorder \right) = \tilde{O}(m^{2/3}).
    $$
    
\end{restatable}

\paragraph{Step 3 (Bounding the revenue loss due to fewer adoptions).}
    In the previous step, we upper bounded the potential revenue loss for the $\gim$  customers that arrive in any epoch $i$ of the algorithm. 
    However it is
    possible that the algorithm reaches the end of time horizon $T$ early and never even get to observe
    some customers, i.e. $d_T < \Xopt_Tm$. In that case the algorithm would incur an additional regret due to fewer adoptions. Therefore, we need to
    show that our total number of adoptions in time $T$ cannot be  much
    lower than that in the optimal trajectory. To show this, we use the observation made in Step 1 that in each epoch the algorithm offered a lower
    confidence bound on the optimal prices (note that the exploration price $p_0=0$
    is always a lower bound on the optimal prices). 
    Due to the use of lower prices, we can show that with high probability our final number of adoptions can be at most  $\tilde{O}(\sqrt{m})$ below the optimal number of adoptions in the deterministic Bass model. 
    Further, we can prove an $O(\log(T))$ upper bound on the optimal price, which allows us to easily bound the revenue
    loss that may result from the small gap in number of adoptions.
    Lemma~\ref{lem:finalstochXclosetooptimal} and
    Lemma~\ref{lem:priceupperbound} make these results  precise.
\begin{restatable}{lemma}{finalstochXclosetooptimal}
    \label{lem:finalstochXclosetooptimal}
    If the seller follows Algorithm~\ref{alg:main}, then with probability at
    least $1-\delta\log(m)$, the number of adoptions at the end of time horizon $T$ is lower bounded as: 
    $$ \textstyle d_T\geq m\Xopt_T -
    \sqrt{8m\Xopt_T\log(\frac{4}{\delta})}.$$
\end{restatable}
\begin{restatable}[Upper bound on optimal prices]{lemma}{priceupperbound}
    \label{lem:priceupperbound}
    All prices in the optimal price curve for deterministic Bass model are upper bounded as:
    $$
    \pcurve(x, \alpha, \beta)\leq \log\left(e+(\alpha+\beta)T\right)
    $$
\end{restatable}
These two results combined tell us that, compared to the optimal, we lose at
most $\tilde O(\sqrt{m}\log(T))$ revenue due to fewer adoptions. The dominant
term in regret comes from the seller losing roughly $\tilde O(m^{-1/3})$ revenue on each
customer, which led to $\tilde{O}(m^{2/3})$ revenue loss per epoch in Step 2. 

\paragraph{Step 4 (Putting it all together).} 
    Finally, we  put the previous steps together to prove Theorem \ref{thm:main}. Since the number of customers in each epoch 
    grows geometrically 
    and there are at most $m$ customers in total, the number of epochs is bounded by
    $O(\log(m))$. By Step 2, the algorithm loses at most $\approx m^{2/3}$
    revenue on adoptions in each epoch. By Step 3, it loses at most $\approx \sqrt{m} \log(T)$ revenue due to missed adoptions.  The total regret is therefore  bounded by $\approx (\log(m)m^{2/3}+\sqrt{m}\log(T))$.

\section{Regret Lower Bound}
\label{sec: regret lower bound}
We prove a lower bound on the regret of any dynamic pricing and learning algorithm under the following mild assumptions on algorithm design. 
\begin{assumption}
    \label{assum: const price in between arrivals}
    The algorithm can change price only on arrival of a new customer. The price is held constant between two arrivals. 
\end{assumption}

\begin{assumption}
    \label{assum:price upper bound}
    Given a planning horizon $T$, the price offered by the algorithm
    at any time $t\in [0,T]$ is upper bounded by $p_{max} \coloneqq \log(T) + C(\alpha, \beta)$, for some function $C$ of $\alpha, \beta$.
\end{assumption}
The above assumptions are indeed satisfied by Algorithm~\ref{alg:main} since it changes prices only on arrival of a new customer, and the prices offered are clamped to the range $[0,\log(e+\phi T)]$ (refer to \eqref{eq: lcb on opt price}) where $\phi$ is a constant upper bound on $\alpha+\beta$. These assumptions are also satisfied by the optimal dynamic pricing policy for the deterministic Bass model since the optimal prices are bounded by $\log(e+(\alpha+\beta)T)$ (refer to Lemma \ref{lem:priceupperbound}). Note that since customer arrivals are continuous in the deterministic Bass model, Assumption \ref{assum: const price in between arrivals} is vacuous in that setting.


We argue that these assumptions do not significantly handicap an algorithm and preserve the difficulty of the problem. 
Assuming an upper bound on the price is a common practice in dynamic pricing literature. Indeed, unlike most existing literature which assumes a constant upper bound, Assumption \ref{assum:price upper bound} allows the price to potentially grow with the planning horizon. Intuitively, given enough time to sell the product, the seller should be able to potentially increase prices in exchange for a slower
adoption rate.\footnote{
In Lemma~\ref{lem: const price ub and large T makes problem trivial} in Appendix~\ref{sec: auxiliary results}, we show that if a constant upper bound is required on the price, then  the problems becomes trivial for any $T\geq \Omega(\log(m))$).
}  Such a dynamics is observed in the optimal dynamic pricing policy for deterministic Bass model, where the price can grow with the time horizon. However, 
the optimal prices are still uniformly upper bounded by $\pu$.

Furthermore, in the proof of Lemma \ref{lem: pseudo regret and regret} (specifically, refer to  Lemma~\ref{lem:smalloptimalgap} in Appendix~\ref{sec: concavity related proofs}), we show that there exist pricing policies in the stochastic Bass model that satisfy both the above assumptions and achieve an expected revenue that is at most $\tilde{O}(\sqrt{m})$ additive factor away from the deterministic optimal revenue $\detV(T)$. Since the lower bound provided in this section is of the order $m^{2/3}$, this indicates that removing these assumptions from algorithm design is unlikely to provide an advantage significant enough to overcome the lower bound derived here.


Our main result in the section is the following
regret lower bound on any algorithm under the above assumptions.
\begin{restatable}[Regret lower bound]{theorem}{thmmainlowerbound}
    Fix any $\alpha>0,
    \beta>0$, and $T=\frac{2(1+\sqrt{2})e}{\alpha+\beta}$. Then, given any pricing algorithm satisfying Assumption \ref{assum: const price in between arrivals} and \ref{assum:price upper bound}, there exist  Bass model parameters $(\alpha,\beta')$ with 
    $\beta' \in [\beta,
    \beta+\frac{(\alpha+\beta)^2}{\alpha}]$ such that the expected regret of the algorithm in this market has the following lower bound:
    $$\textstyle \bbE[\preg(T)]\geq
    \Omega\left(\left(\frac{\alpha}{\alpha+\beta}\right)^{4/3}m^{2/3}\right).$$
    Here the expectation is taken with respect to the stochasticity in the
    Bass model as well as any randomness in the algorithm.
    \label{thm:mainlowerbound}
\end{restatable}

Note that given the relation between pseudo-regret and regret in Lemma \ref{lem: pseudo regret and regret}, the above theorem directly implies an $\Omega(m^{2/3})$ lower bound on $\bbE[\reg(T)]$.

\paragraph{\underline{Lower bound implications}} 
Theorem \ref{thm:mainlowerbound}  highlights the regime of horizon $T$ where this learning problem is most challenging. In this problem, we observe that if $T$ is large, the seller can simply offer a high price for the entire time horizon and still capture most of the market, making the problem trivial. Specifically, for $T\geq \Omega(\log(m))$, under an additional assumption that there is a constant upper bound on price, we can show that an  $O(\log(m))$ upper bound on regret can be trivially achieved by offering the maximum price at all times (see Lemma~\ref{lem: const price ub and large T makes problem trivial}).
On the other hand, when $T=o(1)$, then we can show that achieving a sublinear (in $m$) regret is trivial for any algorithm. This is because from \eqref{eq: intro XT_model_quadratic_sol1} we have that in this case the optimal number of adoptions $\Xopt_T\leq \frac{\alpha+\beta}{e}T = o(1)$. Furthermore, from Lemma~\ref{lem:priceupperbound} we know that all prices in the optimal curve can be bounded by $\pcurve_{max} = O(1)$ in this case. This means that the optimal revenue is at most $m\Xopt_T\pcurve_{max} = o(m)$, i.e., sublinear in $m$.
 Intuitively, for such a small $T$, the word-out-mouth effect never comes into play (i.e., the $\alpha+\beta x$ term in the Bass model is dominated by $\alpha$), making the problem uninteresting. 
Our lower bound result therefore pinpoints the exact order of $T$, i.e., $T=\Theta(1/(\alpha+\beta))$, where the difficult and interesting instances of this problem come from.

In the rest of this section, we describe the proof of Theorem \ref{thm:mainlowerbound}. We first provide an intuition for our lower bound and then go over the proof outline. 

\paragraph{\underline{Proof Intuition}} We start with showing that the pseudo-regret of any algorithm can be lower bounded in terms of its cumulative pricing errors. Therefore, in order to achieve low regret, any algorithm must be able to estimate optimal prices accurately.

Next, we observe that in this problem, the main difficulty for any algorithm comes from the difficulty in estimating the market parameter $\beta$. Since $\beta$'s observed influence on the arrival rate of customers is proportional to the current adoption level $x$, one cannot estimate
$\beta$ accurately when $x$ is small. This makes intuitive sense because when the
number of adopters is small, we do not expect to be able to measure the
word of mouth effect accurately. In fact, we demonstrate that for any $\varepsilon$, before
the adoption level exceeds
$\left(\frac{m}{\varepsilon}\right)^{2/3}$, no algorithm can 
distinguish two markets with Bass model parameters  $(\alpha, \beta)$ vs.
$(\alpha, \beta+\varepsilon)$. 
Further, we can show that in some problem instances (specifically for instances with $T=\Theta(1/(\alpha+\beta))$), 
the optimal prices for the first $m^{2/3}$ customers  are very sensitive to the value of $\beta$. 
Therefore, if an algorithm's 
estimation of $\beta$ is not accurate, it cannot accurately compute the optimal price for these customers. 
This presents an impossible
challenge for any algorithm: it needs an accurate estimate of $\beta$ in order to compute an accurate enough optimal price for the first $m^{2/3}$ customers, but
it cannot possibly obtain such an accurate estimate while the adoption level is that low; thus, it must incur pricing errors resulting in the given lower bound on regret. 

\paragraph{\underline{Proof Outline}}
A formal proof of Theorem \ref{thm:mainlowerbound} is obtained through the following four steps.
All the missing proofs  from this section can be found in Appendix~\ref{sec: lower
bound proofs}. 

\paragraph{Step 1.} First we show that the
pseudo-regret of an algorithm  can be lower bounded in terms of cumulative
pricing errors made by the algorithm. Note that this result is not apriori obvious because as we have
discussed earlier, prices have long term effects on the adoption
curve: the immediate loss of revenue in the current round by offering too low
of a price might be offset by the fact that we saved some time for the future
rounds (lower price means faster arrival rate). On the other hand, if we offer a
price that is higher than the optimal price, the resulting delay in customer arrival (higher price means slower arrival rate) may lead to less remaining time and fewer adoptions, which could be more harmful than the immediate extra revenue. 
The result proven here 
is crucial for precisely quantifying these tradeoffs and lower bounding the regret in terms of pricing errors. 

To obtain this result, we first lower bound the impact of offering a suboptimal price at any time, on the remaining value in the deterministic Bass model. Given any current adoption level $x$ and remaining time $T'$, the impact or `disAdvantage' $A^{det}(x,T', p)$ of price $p$ is defined as the total decrease in value over the remaining time when $p$ is offered for an infinitesimal time and then optimal policy is followed. That is,
$$ A^{det}(x, T', p) :=  \lim_{\delta \rightarrow 0} \frac{\detV(x, T')-p\lambda(p, x)\delta - \detV(x+\lambda(p, x)\delta/m, T'-\delta)   }{\delta}.$$
We prove the following lemma lower bounding this quantity.
\begin{restatable}{lemma}{instantaneousimpactofsuboptimalprice}
\label{lem: instantaneous impact of suboptimal price}
At any adoption level $x$ and remaining time $T'$, the disadvantage of offering a suboptimal price $p$ 
in the deterministic Bass model is lower bounded as:
$$ A^{det}(x,T',p)\geq \lambda(p, x) \min\left( \frac{1}{4}(\popt(x, T')-p)^2, \frac{1}{10}\right).$$
where $\popt(x, T')=\arg \min_p A^{det}(x,T',p)$ denotes the optimal price at $x,T'$.
\end{restatable}
\removedshipra{
Note that given \emph{any} adoption level $x$ and remaining time $T'$, the optimal policy will offer $\popt(x, T')$. This is distinct  from $\pcurve(x, \alpha, \beta)$ defined in \eqref{eq: opt_p_model_in_opt_trajectory} because $\pcurve(x, \alpha, \beta)$ is the price trajectory if the optimal policy $\popt$ is followed from $t=0, X_0=0$. We use $\popt(x, \alpha, \beta, T')$ to emphasize the optimal policy's dependence on the market parameters.  More details on this distinction between $\pcurve$ and $\popt$ can be found in Appendix~\ref{sec: det opt price properties}.}
Now recall that pseudo-regret is defined as the total difference in revenue of the algorithm, which is potentially offering suboptimal prices,  compared to the deterministic optimal revenue $\detV(T)$. We use the above result to quantify the impact of offering suboptimal prices for the first $n\approx m^{2/3}$ 
customers, along with a bound on difference in stochastic vs. deterministic optimal revenue, to  obtain the following lower bound on the pseudo-regret:

$$\bbE\left[\preg(T)\right] \ge \bbE\left[m\int_{0}^{n/m}
    \min\left(\frac{1}{4}(\popt_x-p_x)^2,
    \frac{1}{10}\right) dx\right]  -\tilde{O}\left(m^{1/3}\right).$$
Here, $p_x$ denotes the price trajectory obtained on using the algorithm's prices in the deterministic Bass model, and $\popt_x$ denote the prices that would minimize the disadvantage at  each  point in this trajectory.  A more precise statement of this result is in Lemma \ref{lem: lower bound pseudo regret price difference} in Appendix \ref{sec: lower bound proofs}.

The remaining proof focuses on lower bounding the cumulative difference  in pricing errors $(\popt_x-p_x)^2$ that any algorithm must make for the first $n \approx m^{2/3}$ customers.

\paragraph{Step 2.} Consider two markets with parameters $(\alpha, \beta)$,
and $(\alpha, \beta+\epsilon)$ for some constant $\epsilon$. We show that for the first $n \approx \left(\frac{m}{\epsilon}\right)^{2/3}$ customers,  any pricing
algorithm will be ``wrong'' with a constant probability. Here
by being wrong we mean that the algorithm will set a price that is closer to the optimal price of the other market.  
Lemma~\ref{lem: decision rule limit} and Corollary \ref{cor:decision rule limit} formalize this idea. The proof is based on a standard  information theoretic analysis using KL-divergence.

\paragraph{Step 3.} Next we show that for $T=\Theta(\frac{1}{\alpha+\beta})$, the difference
between the optimal prices for the two markets ($(\alpha, \beta)$ and
$(\alpha, \beta+\varepsilon)$) is large ($\approx \frac{\alpha
\varepsilon}{(\alpha+\beta)^2}$). 
We show this by proving a bound on the derivative of optimal price with respect to $\beta$. 
Lemma~\ref{lem: optimal price difference
deterministic} in Appendix \ref{sec: lower bound proofs} gives the precise statement.

\paragraph{Step 4.} Finally, we put together the observations made in the previous three steps to prove Theorem \ref{thm:mainlowerbound}. We have shown that with constant probability any algorithm will make
a pricing mistake for the first $\approx (\frac{m}{\epsilon})^{2/3}$ customers [Step 2] and that this mistake will be large (on the order of $\epsilon$) [Step 3] under
the condition that $T=\Theta(1/(\alpha+\beta))$. We also have a lower bound on pseudo-regret in terms of total (square of) pricing errors made over the first $n \approx m^{2/3}$ customers [Step 1]. Combining these observations with an appropriately chosen $\epsilon$ gives us that the
pseudo-regret is lower bounded by 
$O(\epsilon^2 (\frac{m}{\epsilon})^{2/3}) \approx O(m^{2/3})$. 

All the missing details of this proof can be found in Appendix \ref{sec: lower bound proofs}.



\section{Conclusions}
In this paper we investigated a novel formulation of dynamic pricing and learning, with a non-stationary demand process  governed by an unknown stochastic Bass model. Deriving an optimal policy in this stochastic model is challenging even when the model parameters are known. Here, we presented an online algorithm that learns to price from past observations without a priori knowledge of the model parameters. A key insight that we derive and utilize in our algorithm design is the {\it concavity} of the optimal value in the deterministic Bass model setting. Using this concavity property, we can show that the optimal value in the deterministic Bass model is always higher than in the stochastic model, and therefore can be used as a stronger benchmark to compete with. Based on this insight, the main  algorithmic idea is to follow the well-known optimal price curve for the deterministic model but with estimated model parameters substituted in place of their true values. 

Our main technical result is an upper bound of $\tilde O(m^{2/3})$ on the regret of our algorithm in a market of size $m$, along with a matching lower bound of $\Omega(m^{2/3})$ under  mild restrictions on algorithm design. 
Thus, our algorithm has close to the best performance achievable by any algorithm for this problem. The derivation of our lower bound is especially involved, and requires deriving  novel dynamic-programming based inequalities. These  allow for lower bounding the loss in long-term revenue in terms of instantaneous pricing errors made by any non-anticipating algorithm. 

An interesting and perhaps surprising aspect of our upper and lower bounds is the role of the horizon $T$ vs. market size $m$. Our upper bound depends sublinearly on $m$ but only logarithmically on the horizon $T$. And in fact our lower bound indicates that for any fixed $\alpha, \beta$ the most interesting (and challenging) instances of this problem are characterized by $T$ which is of constant order, and large $m$. This insight highlights the distinct nature of pricing under stateful models like the Bass model when compared to the independent demand models and multi-armed bandit based formulations where asymptotics with respect to $T$ form the main focus of the analysis. Interesting directions for future research include investigation of other stateful demand models where the concavity property and other new dynamic programming based insights derived here may be useful.

\printbibliography


\newpage
\appendix
\section{A Concentration Result on Exponential Random Variables}
\label{sec: concentration result}
Before moving on to the next sections, we state a concentration
result that is used in many of the remaining proofs.
\begin{lemma}
    Let $[I_d, d=1,\ldots, n]$ be a sequence of random variables with filtration $\cF_{d}$ such that
    $I_d | \cF_{d-1}$ is an exponential random variable with rate $\lambda_d
    \in \cF_{d-1}$. Suppose that there
    exists $\underline{\lambda}>0$ such that $\lambda_d\geq
    \underline{\lambda}$ almost surely. Then for any $\epsilon\leq
    \frac{2n}{\underline{\lambda}}$,
    $$
    	\bbP\left(\sum\limits_{d=1}^{n} I_d - \frac{1}{\lambda_d} \leq -\epsilon \right)
    	\leq exp\left(\frac{-\epsilon^2\underline{\lambda}^2}{8n}\right)
    $$
    $$
        \bbP\left(\sum\limits_{d=1}^{n} I_d - \frac{1}{\lambda_d} \geq \epsilon \right)
    	\leq exp\left(\frac{-\epsilon^2\underline{\lambda}^2}{8n}\right)
    $$
	\label{lem: general exponential arrival time martingale concentration bound}
\end{lemma}
\begin{proof}
    An exponential random variable with rate $\lambda$ satisifies 
    (see \cite{binomial_locally_subgaussian})
    \begin{equation}
        \label{eq: MGF of exponential}
        \bbE[e^{sX}] =\frac{1}{1-s/\lambda} \leq e^{\frac{s}{\lambda}+
        2\frac{s^2}{\lambda^2}} \quad \forall s\in [-\frac{\lambda}{2},
        \frac{\lambda}{2}]
    \end{equation}
    \begin{align*}
    	\bbP\left(\sum\limits_{d=1}^{n} I_d - \frac{1}{\lambda_d} \leq -\epsilon \right)
    	&=\bbP\left(e^{s\sum\limits_{d=1}^{n} \frac{1}{\lambda_d}-I_d} \geq e^{s\epsilon} \right)\\ 
    	\text{(Markov's Inequality)}&\leq\frac{1}{e^{s\epsilon}}\bbE\left[e^{s\left(\sum\limits_{d=1}^{n} \frac{1}{\lambda_d}-I_d\right)} \right]\\ 
    	&\leq\frac{1}{e^{s\epsilon}}\bbE\left[e^{s\left(\sum\limits_{d=1}^{n-1} \frac{1}{\lambda_d}-I_d\right)}\bbE\left[e^{s(\frac{1}{\lambda_n}-I_{n})} \big\vert \mathcal{F}_{n-1}\right]\right]\\ 
    	(\text{Using \eqref{eq: MGF of exponential} and  $\lambda_n\geq \underline{\lambda}$}) &\leq\frac{1}{e^{s\epsilon}}\bbE\left[e^{s\left(\sum\limits_{d=1}^{n-1} \frac{1}{\lambda_d}-I_d\right)} e^{2\frac{s^2}{\underline{\lambda}^2}}\right]\\ 
    	(\text{Repeating the above argument})&=e^{s^2\frac{2n}{\underline{\lambda}^2} - s\epsilon}\\
    	\text{(solve for optimal $s=\frac{\epsilon
    \underline{\lambda}^2}{4n}$)}&=exp\left(\frac{-\epsilon^2\underline{\lambda}^2}{8n}\right)
    \end{align*}
    Note that to use \eqref{eq: MGF of exponential} in the last
    inequality we needed that $s = \frac{\epsilon \underline{\lambda}^2}{4n}
    \leq \frac{\lambda_d}{2}$ for every $d$, which is satisfied if $\epsilon\leq
    \frac{2n}{\underline{\lambda}}$.
    $\bbP\left(\sum\limits_{d=1}^{n} I_d - \frac{1}{\lambda_d} \geq \epsilon
    \right)$ can be bounded similarly. 
\end{proof}

\section{Some Properties of the Deterministic Optimal Price Curve}
\label{sec: det opt price properties}
\subsection{Optimal pricing policy expression}

The expressions in \eqref{eq: opt_p_model_in_opt_trajectory} \eqref{eq: intro XT_model_quadratic_sol1} and \eqref{eq: intro XT_model_quadratic_sol2} \newshipra{provided the expressions for the optimal price curve $\pcurve(x,\alpha, \beta), x\in [0, \Xopt_T]$ and for the total number of adoptions $\Xopt_T$  when optimal pricing policy is followed from an initial adoption level $X_0=0$ at $t=0$ to the end of the time horizon $T$. Here we derive a more general expression for the optimal pricing policy $\popt(x, \alpha, \beta, T)$ that will give the optimal price at any current adoption level $x$ and \emph{remaining} planning horizon $T$ (irrespective of what pricing policy was followed for how much time to reach the adoption level $x$). Also, we derive an expression for $\Xopt_{T}(x)$, the adoption level at the end of time $T$ if optimal pricing policy is followed for time $T$ starting from the adoption level $X_0=x$ at $t=0$. These expressions will be especially useful in our lower bound derivations.}

\removedshipra{time $T$ in the future in the optimal trajectory, when the current adoption level is $x$ and the remaining planning horizon is $T$. }
Note that under this expanded notation, $\Xopt_T = \Xopt_T(0)$. We sometimes also use the notation $\Xopt_T(x, \alpha, \beta)$ to emphasis the dependence $\Xopt_T$ has on the market parameters. Note that by this definition, $\Xopt_{T-t}(X_t) = \Xopt_T(X_0)$ \newshipra{if $X_t$ is the adoption level reached at time $t$ on following the optimal price trajectory from time $0$.}

The optimal price to offer at any given adoption level $x$ and remaining time $T$ can be derived using optimal control theory  (see
Equation (8) of \cite{experience_curve}).
and is given by the following pricing policy. 
Given adoption  $x\in [0, 1)$ and remaining time $T>0$, the optimal price is given by 
\begin{equation}
    \label{eq:opt_p_model}
		\popt(x, \alpha, \beta, T) := 1
		+ \log\left( \frac{(\alpha+\beta x)(1-x)}{(\alpha+\beta \Xopt_{T}(x) )(1-\Xopt_{T}(x))}\right)
\end{equation}
\removedsy{where $\Xopt_{T}(x)$ is the adoption level at the end of horizon $T'$ when the optimal pricing policy is followed for time horizon $T'$ starting from adoption level $x$.}

When the value of $\alpha, \beta$ is clear from the context, we sometimes drop the dependence on $\alpha,\beta$, and use shorter notations $\pcurve(x)$ and $\popt(x, T)$ instead of $\pcurve(x,\alpha, \beta)$ and $\popt(x, \alpha, \beta, T)$ respectively.

To see the connection (and distinction) between $\pcurve(x)$ and $\popt(x,T)$, note that $\pcurve(x)$ is the price trajectory if the optimal policy $\popt$ is followed from $t=0, X_0=0$ to the end of horizon $T$. 
That is, if $X_t, p_t$ denotes the adoption level and price at time $t \in [0,T]$ on following optimal pricing policy from $t=0, X_0=0$, then
$p_t=\pcurve(X_t,\alpha, \beta)=\popt(X_t, \alpha, \beta, T-t)$.

Now consider the adoption process on starting from an initial adoption level $x$ at time $t=0$, and then following the optimal pricing policy. Again $X_t$ denotes the adoption level at time $t$ in this process.
Plugging $p_t=\popt(X_t, \alpha, \beta, T-t)$ back into \eqref{eq:detrministicBass}, it's easy to derive that 
\begin{equation}
    \label{XT_model}
    \frac{d X_t}{d t} = \frac{1}{e}(\alpha+\beta \Xopt_{T}(x))(1-\Xopt_{T}(x)) 
\end{equation}
This means that under the optimal policy, the rate of adoption is constant. 
We can integrate the above from $t=0$ to $T$, and also solve the resulting quadratic equation for $t=T$ to compute the final
adoption level $\Xopt_T(x)$ under the optimal pricing policy when starting from an initial adoption level $x$ at $t=0$:
\begin{align}
        \Xopt_T(x)-x = \frac{1}{e}(\alpha+\beta \Xopt_T(x))(1-\Xopt_T(x))T\label{eq: general XT_model_quadratic_sol1}\\
        \Xopt_T(x)=\frac{T(\beta-\alpha)-e+\sqrt{[T(\beta-\alpha)-e]^2+4\alpha\beta T^2 + 4e\beta x T}}{2\beta T}\label{eq: general XT_model_quadratic_sol2}
\end{align}
Note that on plugging $x=0$ in \eqref{eq: general XT_model_quadratic_sol2}, we obtain the expression for $\Xopt_T$ in \eqref{eq: intro XT_model_quadratic_sol2}.



\subsection{Price Lipschitz Bound}
We start with some Lipschitz bounds on how the optimal price offered at
adoption level $x$ can change with respect to $\alpha, \beta$. Note that here
we are assuming $X_0=0$ and that the entire optimal price curve from the beginning 
changes if we change $\alpha, \beta$ (i.e., we use \eqref{eq:
opt_p_model_in_opt_trajectory} not \eqref{eq:opt_p_model}). 
\newshipra{In the following lemma, $\Xopt_T$ denotes the adoption curve on following the optimal pricing policy for all times $t\in [0,T]$ starting from $0$ adoption level under the deterministic Bass model with parameters $(\alpha, \beta)$ as given in \eqref{eq: intro XT_model_quadratic_sol2}.} Using the expanded notation introduced in the previous subsection, it can also be called $\Xopt_T(0, \alpha, \beta)$.
\begin{lemma}
    \label{lem: price lipschitz helper result}
    $0\leq\frac{1}{1-\Xopt_T}\frac{\partial \Xopt_T}{\partial \alpha} \leq \frac{1}{\alpha}$, and 
    $0\leq\frac{1}{1-\Xopt_T}\frac{\partial \Xopt_T}{\partial \beta} \leq \frac{1}{\beta}$
\end{lemma}
\begin{proof}
    Below, we use $X_T$ instead of $\Xopt_T$ to denote the adoption
    at time $T$, in the deterministic optimal
    trajectory starting at $X_0=0$. 
    Differentiating \eqref{eq: intro XT_model_quadratic_sol1} with respect to $\alpha$:
    \begin{align*}
        \beta T X_T^2 + eX_T-\beta TX_T + \alpha TX_T - \alpha T &= 0\\
        2\beta T \frac{\partial X_T}{\partial \alpha}X_T + e\frac{\partial X_T}{\partial \alpha} - \beta T\frac{\partial X_T}{\partial \alpha} + TX_T + \alpha T\frac{\partial X_T}{\partial \alpha} - T &= 0\\
        \frac{\partial X_T}{\partial \alpha} (2T X_T\beta + e -\beta T + \alpha T) &= T(1-X_T)
    \end{align*}
    Rearranging and substituting $X_T$ using (\ref{eq: intro XT_model_quadratic_sol2}):
    \begin{align*}
        \frac{1}{1-X_T}\frac{\partial X_T}{\partial \alpha} &= \frac{T}{2\beta TX_T + e + (\alpha-\beta)T}\\
        &=\frac{T}{T(\beta - \alpha) - e + \sqrt{4 T^{2} \alpha \beta + \left(T \left(\beta -\alpha\right) - e\right)^{2}} + (\alpha-\beta)T + e}\\ 
        &=\frac{T}{\sqrt{4 T^{2} \alpha\beta + \left(T(\beta-\alpha) - e\right)^{2}}} \ge 0.\\ 
    \end{align*}
    Note the denominator can be bounded in two ways:
    \begin{align*}
    \sqrt{4 T^{2} \alpha\beta + \left(T(\beta-\alpha) - e\right)^{2}} \geq \left|T(\beta-\alpha)-e\right| \text{ and,}\\
    \sqrt{4 T^{2} \alpha\beta + \left(T(\beta-\alpha) - e\right)^{2}} = \sqrt{T^2(\alpha+\beta)^2 + 2(\alpha-\beta)T + e^2} \geq \left|T(\alpha+\beta)-e\right|.
    \end{align*}
    So 
    $$\frac{1}{1-X_T}\frac{\partial X_T}{\partial \alpha} \leq 
    \min( |\frac{1}{\alpha+\beta-\frac{e}{T}}|, |\frac{1}{\alpha-\beta+\frac{e}{T}}|) =
    min( |\frac{1}{\alpha+(\beta-\frac{e}{T})}|, |\frac{1}{\alpha-(\beta-\frac{e}{T})}|) \leq
    \frac{1}{\alpha}$$


    Following a similar procedure, 
    (differentiating \eqref{eq: intro XT_model_quadratic_sol1} with respect to $\beta$
    and using \eqref{eq: intro XT_model_quadratic_sol2}) we can also get the following bound:
    \begin{align*}
        \frac{1}{1-X_T}\frac{\partial X_T}{\partial \beta} &= \frac{TX_T}{2\beta TX_T + e + (\alpha-\beta)T}\\
        &=\frac{T(\beta-\alpha)-e}{2\beta \sqrt{4 T^{2} \alpha\beta + \left(T \left(- \alpha + \beta\right) - e\right)^{2}}} + \frac{1}{2\beta}\\ 
        0\leq\frac{1}{1-X_T}\frac{\partial X_T}{\partial \beta}&\leq \frac{1}{\beta}
    \end{align*}
    where in the last step we used the fact that $
    \frac{|T(\beta-\alpha)-e|}{ \sqrt{4 T^{2} \alpha\beta + \left(T \left(-
    \alpha + \beta\right) - e\right)^{2}}} \leq 1$.
\end{proof}
In our proofs we will sometimes need  to compare the optimal final adoption level under two different market parameters. To derive the results comparing two different market parameters, instead of $\Xopt_T(x)$, we use the notation $\Xopt_T(x, \alpha, \beta)$ that makes the dependence on $\alpha, \beta$ explicit.
\begin{corollary}
$\Xopt_T(0,\alpha, \beta) \geq \Xopt_T(0, \alpha, 0) = \frac{\alpha T}{\alpha T+e}$.
\label{cor: beta improves final XT}
\end{corollary}
\begin{proof}
    The inequality follows from the non-negativity of $\frac{\partial \Xopt}{\partial \beta}$ proved in Lemma~\ref{lem: price lipschitz helper result}. The equality can be easily derived by solving for $\Xopt_T$ in \eqref{eq: intro XT_model_quadratic_sol1} after plugging in $\beta=0$.
\end{proof}

\begin{restatable}{lemma}{pricelipschitzboundcustomer}
    \label{lem: opt price lipschitz bound} For any $0\le x\le \Xopt_T$,
    $\left|\frac{\partial
    \pcurve(x, \alpha, \beta)}{\partial \alpha}\right|\leq
    \frac{2+\beta/\alpha}{\alpha, } , \left|\frac{\partial \pcurve(x,\alpha,
    \beta)}{\partial \beta}\right| \leq \frac{3}{\min(\alpha,\beta)}.$
\end{restatable}
\begin{proof}
    Below, we use $X_T$ instead of $\Xopt_T$ to denote the optimal final adoption level, assuming that the initial adoption level is $0$.
    Taking the derivatives of \eqref{eq:
    opt_p_model_in_opt_trajectory}, and by using Lemma~\ref{lem: price
    lipschitz helper result}:
    \begin{align*}
        \left|\frac{\partial \pcurve(x, \alpha, \beta)}{\partial \alpha}\right| &= \left|\frac{\partial \log(\alpha+\beta x)}{\partial \alpha} + \frac{\partial \log(1-x)}{\partial \alpha} - \frac{\partial \log(\alpha + \beta X_T)}{\partial \alpha}    -\frac{\partial \log(1-X_T)}{\partial \alpha}\right| \\
        &= \left| \frac{1}{\alpha+\beta x}  - \frac{1}{\alpha+\beta X_T}(1+\beta\frac{\partial X_T}{\partial \alpha}) + \frac{1}{1-X_T}\frac{\partial X_T}{\partial \alpha} \right|\\
        &\leq \max\left(\left| \frac{1}{\alpha+\beta x} +\frac{1}{1-X_T}\frac{\partial X_T}{\partial \alpha} \right|, \left|\frac{1}{\alpha+\beta X_T}(1+\beta\frac{\partial X_T}{\partial \alpha})\right|\right)\\
        &\leq \max\left(\frac{2}{\alpha}, \frac{1}{\alpha}+\frac{\beta/\alpha(1-X_T)}{\alpha}\right)\\
        &\leq \frac{2+\beta/\alpha}{\alpha}
    \end{align*}
    
    \begin{align*}
        \left|\frac{\partial \pcurve(x, \alpha, \beta)}{\partial \beta}\right| &= \left|\frac{\partial \log(\alpha+\beta x)}{\partial \beta} - \frac{\partial \log(\alpha + \beta X_T)}{\partial \beta}    -\frac{\log(1-X_T)}{\partial \beta}\right| \\
        &= \left| \frac{x}{\alpha+\beta x} - \frac{1}{\alpha+\beta X_T}(X_T+\beta\frac{\partial X_T}{\partial \beta}) + \frac{1}{1-X_T}\frac{\partial X_T}{\partial \beta} \right|\\
        &\leq \left| \frac{x}{\alpha+\beta x} - \frac{1}{\alpha+\beta X_T}(X_T+\beta\frac{\partial X_T}{\partial \beta}) \right| + \left|\frac{1}{1-X_T}\frac{\partial X_T}{\partial \beta}\right|\\
        &\leq \frac{1}{\alpha}(1+\beta \frac{\partial X_T}{\partial \beta})+ \frac{1}{\beta}\\
        &\leq \frac{2}{\alpha} + \frac{1}{\beta}\\
        &\leq \frac{3}{\min(\alpha, \beta)}
    \end{align*}
\end{proof}

\section{Proof of Lemma \ref{lem: concavity of deterministic value function} and Lemma \ref{lem: pseudo regret and regret}}
\label{sec: concavity related proofs}
We first prove the concavity property of the deterministic optimal value stated as Lemma~\ref{lem: concavity of deterministic value function}.

\concavity*
\begin{proof}
For simplicity of notation, in this proof we use $T$ to denote the remaining time.
    The optimal value function for the continuous deterministic Bass model
    when the remaining time is $T$ can be expressed using the following dynamic programming equation (for all $\delta \ge 0$),
    \begin{align}
        \detV(x, T) &= \max\limits_p p\lambda(p, x)\delta + \detV(x+\lambda(p, x)\delta/m, T-\delta) + o(\delta)
        \label{eq: deterministic value function bellman equation}
    \end{align}
    where $\lambda(p, x) = me^{-p}(\alpha+\beta x)(1-x)$. 
    

    
    Using the Hamilton-Jacobi-Bellman equation for the deterministic Bass model  (see equation (12.8) in \cite{doya2007bayesian}):
    \begin{equation}
        \label{eq: rate of change of deterministic value wrt time}
        \frac{\partial \detV(x, T)}{\partial T} = \max_p p\lambda(p, x) + \frac{\lambda(p,x)}{m} \frac{\partial \detV(x, T)}{\partial x}.
    \end{equation}
    And the optimal price is the price that achieves the maximum in the above expression (see equation (12.9) in \cite{doya2007bayesian}), i.e.,
    \begin{equation*}
        \popt(x, \alpha, \beta, T) = \argmax_p p\lambda(p, x) + \frac{\lambda(p,x)}{m} \frac{\partial \detV(x, T)}{\partial x}
    \end{equation*}

    Solving the above maximization problem gives us an expression for the optimal price at state $x$ with $T$ time left.
    \removedshipra{to check}
    \begin{equation}
        \label{eq: opt det price wrt value function}
        \popt(x,\alpha, \beta, T) = 1-\frac{1}{m}\frac{\partial \detV(x, T)}{\partial x}
    \end{equation}
    
    From \eqref{eq:opt_p_model}, the optimal price is also given by
    $$\popt(x,\alpha, \beta, T)=1+\log\left(\frac{(\alpha+\beta
    x)(1-x)}{(\alpha+\beta \Xopt_{T}(x))(1-\Xopt_{T}(x))}\right)\nonumber\\
    $$
    where (refer to \eqref{eq: general XT_model_quadratic_sol1}, \eqref{eq: general XT_model_quadratic_sol2})
    \begin{align*}
        \Xopt_T(x) = x+\frac{1}{e}(\alpha+\beta \Xopt_T(x))(1-\Xopt_T(x))T\\
        \Xopt_T(x)=\frac{T(\beta-\alpha)-e+\sqrt{[T(\beta-\alpha)-e]^2+4\alpha\beta T^2 + 4e\beta x T}}{2\beta T}
\end{align*}
    Therefore, substituting,
    \begin{align}
    \frac{1}{m}\frac{\partial \detV(x, T)}{\partial x} &= -\log\left(\frac{(\alpha+\beta
    x)(1-x)}{(\alpha+\beta \Xopt_{T}(x))(1-\Xopt_{T}(x))}\right)\nonumber\\
    \frac{1}{m}\frac{\partial^2 \detV(x, T)}{\partial x^2} &=
    \frac{-\beta}{\alpha+\beta x} + \frac{1}{1-x} + \frac{\beta
    \frac{\partial \Xopt_{T}(x)}{\partial x}}{\alpha+\beta \Xopt_{T}(x)} -
    \frac{1}{1-\Xopt_{T}(x)}\frac{\partial \Xopt_{T}(x)}{\partial x}\label{hessian of
    deterministic value function}
    \end{align}

    Now we split \eqref{hessian of deterministic value function} to two parts and bound them by zero individually.
    First note that 
    differentiating $X_T^*(x)$ 
    with respect to $x$ (using \eqref{eq: general XT_model_quadratic_sol1}) gives us 
    \begin{equation}
    \frac{\partial \Xopt_{T}(x)}{\partial x} = \frac{e}{e + (\alpha-\beta)T + 2\beta T\Xopt_{T}(x)}
    \label{dXTdx0}
    \end{equation}
    Then, first we show that
    \begin{align*}
        &&\frac{1}{1-x} &- \frac{1}{1-\Xopt_{T}(x)}\frac{\partial \Xopt_{T}(x)}{\partial x} \leq 0, \\
    \end{align*}
    which is equivalent to showing that
    \begin{align*}
        &&1-\Xopt_{T}(x) &\leq \frac{\partial \Xopt_{T}(x)}{\partial x}(1-x) \\
    \text{Using \eqref{eq: general XT_model_quadratic_sol1}}    &\iff& 1-\Xopt_{T}(x) &\leq    \frac{\partial \Xopt_{T}(x)}{\partial x} ( 1-\Xopt_{T}(x)) + \frac{\partial \Xopt_{T}(x)}{\partial x}\frac{T}{e}(\alpha+\beta \Xopt_{T}(x))(1-\Xopt_{T}(x))\\ 
    \text{Using \eqref{dXTdx0}}   &\iff& 1-\Xopt_{T}(x) &\leq (1-\Xopt_{T}(x))(\frac{e}{e+(\alpha-\beta)T +2\beta T\Xopt_{T}(x)})(\frac{e+T(\alpha+\beta \Xopt_{T}(x))}{e})\\
        &\iff& 1-\Xopt_{T}(x) &\leq (1-\Xopt_{T}(x)) \cdot \frac{e+\alpha T + \beta T\Xopt_{T}(x)}{e+(\alpha-\beta)T +2\beta T\Xopt_{T}(x)}
    \end{align*}
   Using expression for $\Xopt_T(x)$ from \eqref{eq: general XT_model_quadratic_sol2}, we have \begin{center}
       $e+(\alpha-\beta)T+2\beta T
    X^*_T(x)=\sqrt{[T(\beta-\alpha)-e]^2+4\alpha\beta T^2 + 4e\beta x
    T}\geq 0$.
    \end{center}
    Then, since $\beta\geq 0, T\geq 0, 0\le \Xopt_T(x)\leq 1$, we have that the fraction in the last inequality is at least $1$, and therefore the last inequality holds. 

 Now we bound the remaining terms in the RHS of \eqref{hessian of deterministic value function} by $0$. This requires showing that
    \begin{align*}
    &&\frac{-\beta}{\alpha+\beta x} + \frac{\beta \frac{\partial \Xopt_{T}(x)}{\partial x}}{\alpha+\beta \Xopt_{T}(x)} &\leq 0\\ 
    &\iff& (\alpha+\beta x)  \frac{\partial \Xopt_{T}(x)}{\partial x} &\leq \alpha+\beta \Xopt_{T}(x)\\ 
    \text{Using \eqref{eq: general XT_model_quadratic_sol1}}&\iff& \frac{\partial \Xopt_{T}(x)}{\partial x}\left(\alpha+\beta \Xopt_{T}(x) - \beta\frac{T}{e}(\alpha+\beta \Xopt_{T}(x))(1-\Xopt_{T}(x))\right) &\leq \alpha+\beta \Xopt_{T}(x)\\ 
    &\iff& (\alpha+\beta \Xopt_{T}(x))\frac{\partial \Xopt_{T}(x)}{\partial x} (1-\frac{\beta T}{e}(1-\Xopt_{T}(x)))&\leq \alpha+\beta \Xopt_{T}(x)\\ 
    \text{Using \eqref{dXTdx0}}&\iff& \frac{e-\beta T + \beta \Xopt_{T}(x) T}{e+(\alpha-\beta)T +2\beta \Xopt_{T}(x)T} &\leq 1
    \end{align*}
    Since $\alpha, \beta\geq 0, T\geq 0, \Xopt_T(x)\geq 0$, the last inequality holds. 
    
    Therefore the sum of all the terms in the right hand side of \eqref{hessian of deterministic value function} is bounded by $0$. This proves the lemma statement.
\end{proof}

Now we use the above concavity property to show that for any starting point and remaining time, the optimal
revenue in the deterministic model is at least the optimal expected revenue in the stochastic model. Let $\stochV(d, T)$ be the optimal expected revenue one can achieve in the stochastic Bass model with $d$ current adopters and $T$ time remaining.
\begin{restatable}{lemma}{detupperbound}
    \label{lem: deterministic value better than stochastic value} 
    For any $d\in\{0,\ldots, m\}$, $x=\frac{d}{m}$, and any $T\geq 0$:
    $$\detV(x, T) \geq \stochV(d, T)$$
\end{restatable}
\begin{proof}
    Given any $\delta \ge 0$, let $\Delta(p, d)$ be the random number of adoptions that take place in the next $\delta$ time in the discrete stochastic Bass model when the
    current price is $p$ and current number of adopters is $d$. Let
    $x=\frac{d}{m}$, then $\bbE[\Delta(p, x)] = \lambda(p, x)\delta +
    o(\delta)$. Using dynamic programming, we have:
	\begin{align*}
		\stochV(d, T) &= \max\limits_p \bbE\left[ p\Delta(p, d) + \stochV(d+\Delta(p, d), T-\delta)\right] +o(\delta)\\
		&= \max\limits_p  p\bbE[\Delta(p, d)] + \bbE\left[\stochV(d+\Delta(p, d), T-\delta)\right] +o(\delta).
	\end{align*}
	Also,
	\begin{align*}
		\detV(x, T) &= \max\limits_p p\bbE[\Delta(p, d)] + \detV(x+\bbE[\Delta(p, d)]/m, T-\delta) + o(\delta)
	\end{align*}

	We use induction to prove the lemma by working from the end of the
	planning horizon (no time remaining). We know $\detV(x, 0) =
	\stochV(d, 0)=0$ for all $d$, $x=\frac{d}{m}$. Suppose the inequality holds
	for $T-\delta$, and any $d$, $x=\frac{d}{m}$. Then,
	\begin{align*}
		\detV(x, T) &= \max\limits_p p\bbE[\Delta(p, d)] + \detV(x+\bbE[\Delta(p, d)]/m, T-\delta) +o(\delta)\\
		\text{(using concavity from Lemma~\ref{lem: concavity of deterministic value function})}&\geq \max\limits_p p\bbE[\Delta(p, d)]m + \bbE\left[\detV(x+\Delta(p, d)/m, T-\delta)\right] + o(\delta)\\
		\text{(inductive assumption on $T-\delta$)}&\geq \max\limits_p p\bbE[\Delta(p, d)]m + \bbE\left[\stochV(d+\Delta(p, d), T-\delta)\right]+o(\delta)\\
		&= \stochV(d, T)+o(\delta)
    \end{align*}
    Then, taking $\delta \rightarrow 0$, we obtain the lemma statement.
\end{proof}
Finally, we prove the following lemma that will allow us to establish an upper bound on the deterministic optimal revenue compared to the stochastic optimal revenue.
\begin{restatable}{lemma}{smalloptimalgap}
    \label{lem:smalloptimalgap}
    \removedsy{Fix any $d_0\in\{0,\ldots, m\}$, $x_0=\frac{d_0}{m}$, $T\geq 0$ such that 
    \removedsy{$m-\lfloor m\Xopt_T(x_0)\rfloor \geq 2$ and }
    $$m(\Xopt_T(x_0)-x_0)\geq 8\log^2\left(4m +4m\log\left(\frac{2(\alpha+\beta)^2T}{e\alpha}+\frac{2(\alpha+\beta)}{\alpha}\right)\right)+32$$
    then $$\detV(x_0, T) - \stochV(d_0, T)\leq 
    O\left(\log\left(\frac{\alpha+\beta}{\alpha}T\right)\sqrt{m(\Xopt_T(x_0)-x_0)\log(m)}\right)$$}
    Fix any $\alpha, \beta, T$ such that 
    $$m\Xopt_T\geq 8\log^2\left(4m\log(e+(\alpha+\beta)T)\right)+32$$
    then $$\detV(0, T) - \stochV(0, T)\leq 
    O\left(\log\left((\alpha+\beta)T\right)\sqrt{m\Xopt_T\log(m)}\right)$$
\end{restatable}
\begin{proof}
    The proof constructs a fixed price sequence  such that the expected revenue in the stochastic
    Bass model under these prices is within $O(\sqrt{m})$ of the optimal deterministic revenue. Then, since the optimal expected
    stochastic revenue is at least as much as that obtained under the given price sequence, we will obtain the lemma statement. 

    Consider the following pricing scheme: for all time instances after arrival of $(d-1)^{th}$ customer, and until arrival of $d^{th}$ customer, post price $\pcurve_d$ given by:
    $\pcurve_d\coloneqq \pcurve\left(\frac{d-1}{m}, \alpha,
    \beta\right)$
    for
    $d=d_0+1,\ldots, \lfloor m\Xopt_T\rfloor$. Here $\pcurve(\cdot, \alpha, \beta)$ and $\Xopt_T$ are given by the optimal price curve and adoption levels for the deterministic Bass model, as defined in  \eqref{eq: opt_p_model_in_opt_trajectory} and \eqref{eq: intro XT_model_quadratic_sol2}.

    Now, since in our price sequence, the prices are fixed between two arrivals,
    we know that (in the stochastic Bass model) the inter-arrival time between customer $d-1$ and $d$ is an
    exponential random variable $I_d \sim Exp\left(\lambda(\pcurve_d,
    \frac{d-1}{m})\right) $, where $\lambda(p, x) = e^{-p}(\alpha+\beta
    x)(1-x)m$. Note that from \eqref{XT_model} and
    \eqref{eq: general XT_model_quadratic_sol1} we have that $\lambda(\pcurve_d,
    \frac{d-1}{m}) = \frac{m\Xopt_T}{T}$ , and that
    $\sum\limits_{d=1}^{n} \frac{1}{\lambda(\pcurve_d,
    \frac{d-1}{m})} = \frac{nT}{m\Xopt_T}$ for any $n \le \lfloor m \Xopt_T \rfloor$.

    Let $\tau_n$ denotes the time of arrival of the $n^{th}$ customer in the stochastic model.
    Set $n=\lfloor m\Xopt_T - \sqrt{8m\Xopt_T\log(\frac{2}{\delta})}\rfloor$,
    $\underline{\lambda}=\frac{m\Xopt_T}{T}$, 
    $\epsilon = \sqrt{\frac{8n}{\underline{\lambda}^2}log(\frac{2}{\delta})}$,
    for some $\delta \ge 0$ to be specified later. Assume $n\geq
    2\log(2/\delta)$ for now, then by Lemma~\ref{lem: general
    exponential arrival time martingale concentration bound}, we have with
    probability $1-\delta$: \removedshipra{comment to myself:to check}
    $$\left|\tau_{n} -
    \left(1-\sqrt{\frac{8\log(\frac{2}{\delta})}{m\Xopt_T}}\right)T\right| \leq
    T\sqrt{\frac{8}{\Xopt_T m }\log(\frac{2}{\delta})}$$
    Therefore, with probability at least $1-\delta$, $\tau_n \leq T$, which means that
    the total number of adoptions $d_T$ observed in the stochastic Bass model in time $T$ is at
    least $n$. Let $\pcurve_x = \pcurve(x, \alpha, \beta)$,
    $\pcurve_{max}=\max\limits_x \pcurve(x, \alpha, \beta)$. Then, 

    \begin{align*}
        \detV(0, T) - \stochV(0, T) \le 
        & m\int_{0}^{\Xopt_T} \pcurve_x dx - \sum\limits_{d=1}^{d_T}\pcurve_{\frac{d-1}{m}}\\
        \text{(Corollary~\ref{cor: discretization is good enough})}\leq& \sum\limits_{1}^{\lfloor m\Xopt_T \rfloor} \pcurve_{\frac{d-1}{m}}  - \sum\limits_{d=1}^{d_T}\pcurve_{\frac{d-1}{m}} + \frac{\beta}{2\alpha}+\frac{1}{2}\log\left(m\right)+3\pcurve_{max}\\
        \leq& \sum\limits_{n+1}^{\lfloor m\Xopt_T \rfloor} \pcurve_{\frac{d-1}{m}} + \frac{\beta}{2\alpha}+\frac{1}{2}\log\left(m\right)+3\pcurve_{max} \quad\text{w.p. }1-\delta\\
        \text{(Lemma~\ref{lem:priceupperbound})}= & O\left(\log\left((\alpha+\beta)T\right) \sqrt{m\Xopt_T\log(\frac{1}{\delta})}\right)\quad\text{w.p. }1-\delta
    \end{align*}
    where we borrowed Corollary~\ref{cor: discretization is good enough} from Appendix~\ref{sec: per epoch and overall regret bounds}, and used the price upper bound $\pcurve_{max} \le O\left(\log\left((\alpha+\beta)T\right)\right)$ from Lemma \ref{lem:priceupperbound}.

    Note that the third step holds with probability $1-\delta$. When it does not hold (with probability at most $\delta$), we can bound the 
    gap between deterministic and stochastic revenue by a trivial upper bound of $\pcurve_{max}m\Xopt_T$ on the deterministic revenue.
    We set $\delta =\frac{1}{\sqrt{m\Xopt_T \pcurve_{max}}}$ to get that \begin{align*}
        &\detV(0, T) - \stochV(0, T)\\ \leq 
    &  (1-\delta) O\left(\log\left((\alpha+\beta)T\right) \sqrt{m\Xopt_T\log(\frac{1}{\delta})}\removedsy{+\frac{\beta}{\alpha}}\right)  + \delta \pcurve_{max}m\Xopt_T\\
    \le &  
    O\left(\log\left((\alpha+\beta)T\right)\sqrt{m\Xopt_T\log\left(m\right)} \right).
    \end{align*}
    Finally, one can verify that the condition on $m\Xopt_T$ implies that $n\geq 2\log(2/\delta)$. Let $M = m\Xopt_T$, and assume that $\log(2\sqrt{M\pcurve_{max}})\geq 1$:
    \begin{align*}
        n\geq 2\log(2/\delta) \iff & M \geq  2\sqrt{2M\log(2\sqrt{M\pcurve_{max}})} +2\log(2\sqrt{M\pcurve_{max}})\\ 
        \impliedby & M\geq \left(2\sqrt{2M}+2\right)\log(2\sqrt{M\pcurve_{max}})\\
        \impliedby &\sqrt{M}\geq 2\sqrt{2}\log(4M\pcurve_{max})\\
        \text{(Lemma~\ref{lem:priceupperbound})}\impliedby &\sqrt{M}\geq 2\sqrt{2}\log\left(4M\log(e+(\alpha+\beta)T)\right)
    \end{align*}
    If $\log(2\sqrt{M\pcurve_{max}})< 1$, then the first line is implied by $M\geq 2\sqrt{2M} +2$, which is satisfied by $M\geq 32$.

\end{proof}

The proof of Lemma \ref{lem: pseudo regret and regret} can now be completed using the upper and lower bounds on deterministic optimal revenue compared to the stochastic optimal revenue proven above.


\begin{proof}[Proof of Lemma \ref{lem: pseudo regret and regret}]
The first part of Lemma~\ref{lem: pseudo regret and regret} follows directly from Lemma~\ref{lem: deterministic value better than stochastic value} because 
$$\preg(T) \geq \reg(T) \iff \detV(0, T)\geq \stochV(0,T).$$
Similarly, the second part follows from Lemma~\ref{lem:smalloptimalgap}. Note that if the condition on $m\Xopt_T$ does not hold, then the gap $\detV(0, T)-\stochV(0,T)$ can be trivially bounded by $O\left(\log T\log^2\left(m\log T\right)\right)$:
\begin{align*}
    \detV(0,T)-\stochV(0, T) \leq &m\Xopt_T\pcurve_{max} - 0\\ 
    \text{(Lemma~\ref{lem:priceupperbound})}\leq & \left[8\log^2\left(4m\log(e+(\alpha+\beta)T)\right)+32\right] O(\log((\alpha+\beta)T))\\ 
    =& O\left(\log((\alpha+\beta)T)\log^2\left(m\log((\alpha+\beta)T)\right)\right)
\end{align*}

\end{proof}

\section{Upper Bound Proofs}
\label{sec: upper bound proofs}

\subsection{Step 1: Bounding the  estimation errors
(Proof of Lemma~\ref{lem:alphahatbetahatbound}, Lemma~\ref{lem: lcb on opt price})}
\label{sec: estimation bounds}
We prove the estimation bound on $\hat \alpha$ and $\hat \beta_i$ separately in Lemma~\ref{lem:alphahatbound} and Lemma~\ref{lem:betahatbound} respectively. Lemma~\ref{lem:alphahatbetahatbound} follows directly from these two results.
\begin{restatable}{lemma}{alphahatbound}
    \label{lem:alphahatbound}
    Assuming that $m^{1/3} \geq \frac{16
    (\alpha+\beta)}{\alpha}\sqrt{8\log(\frac{2}{\delta})}$,
    then with probability $1-\delta$, $|\alpha-\hat \alpha|
    \leq A$.
\end{restatable}
\begin{proof}
	From
	\eqref{eq:alphahat} we have the following expression for the estimator
	error of $\hat \alpha$:
    $$|\alpha-\hat \alpha| = \gamma e^{\pe} |\frac{1}{\gamma m \bbE[\tau_1]} - \frac{1}{\tau_{\gamma m}}|$$
    Recall that $\tau_d$ denotes the (stochastic) time of arrival of $d^{th}$ customer in the stochastic Bass model under the pricing decisions made by the algorithm. Note that in our algorithm prices do not change between customer arrivals. Therefore, the interarrival time $I_d=\tau_d-\tau_{d-1}$ between $d-1$ and $d$ customer  follows an exponential distribution. 
Specifically, since the prices were fixed as $\pe$ for the first $\gm$ customers, we have $I_d\sim  Exp(\lambda(\pe, \frac{d-1}{m}))$ for $d\in \{1,\ldots,\gamma m\}$ where $\lambda(p,x) = me^{-p}(\alpha+\beta x)(1-x)$ and
    $\lambda(\pe, \frac{d-1}{m}) \geq \underline{\lambda} \coloneqq e^{-\pe}\alpha(1-\gamma)m$ for $d\in \{1,\ldots,\gamma m\}$.

    Therefore, $\bbE[\tau_1] = \bbE[I_1] = \frac{1}{e^{-\pe}\alpha m}$, and $\bbE[\tau_{\gamma m}] = \bbE[\sum_{d=1}^{\gm} I_d]$.  
    Using Lemma~\ref{lem: general exponential arrival time martingale concentration bound} we have:
    \begin{align*}
        \bbP\left( \left|\tau_{\gamma m} - \sum\limits_{d=1}^{\gamma m} \bbE[I_d]\right| \geq \epsilon\right)
		\leq  & 2 \exp(-\frac{\epsilon^2\underline{\lambda}^2}{8\gamma m}),\\
	\text{so that } \left|\tau_{\gamma m} - \sum\limits_{d=1}^{\gamma m} \frac{1}{\lambda(p_0, d)} \right| 
	\leq &\frac{e^{p_0}}{\alpha(1-\gamma)}\sqrt{8\log(\frac{2}{\delta})\frac{\gamma}{m}} \quad \text{with probability } 1-\delta,
    \end{align*}
    where we set $\epsilon = \sqrt{\frac{8\log(\frac{2}{\delta})\gamma m}{\underline{\lambda}^2}}$.
    Since $m^{1/3}\geq \sqrt{2\log(\frac{2}{\delta})} \implies
    \epsilon \leq \frac{2\gamma m}{\underline{\lambda}}$, the condition for
    Lemma~\ref{lem: general exponential arrival time martingale concentration
    bound} is satisfied. To bound the estimation error of $\alpha$:
    \begin{align*}
        \left|\tau_{\gamma m} - \gamma m\bbE[\tau_1]\right| &= \left|\tau_{\gamma m} - \sum\limits_{d=1}^{\gamma m } \frac{1}{\lambda(\pe, d)} \right| + \left|\sum\limits_{d=1}^{\gamma m } \frac{1}{\lambda(\pe, \frac{d-1}{m})} - \gamma m\bbE[\tau_1]\right|\\
        \text{(with probability $1-\delta$)}
		&\leq \frac{e^{\pe}}{\alpha(1-\gamma)}\sqrt{8\log(\frac{2}{\delta})\frac{\gamma}{m}} + \left|\sum\limits_{d=1}^{\gamma m} (\frac{1}{\lambda(\pe, \frac{d-1}{m})} -\frac{1}{e^{-\pe}\alpha m})\right|\\
		&\leq \frac{e^{\pe}}{\alpha(1-\gamma)}\sqrt{8\log(\frac{2}{\delta})\frac{\gamma}{m}} + \left|\sum\limits_{d=1}^{\gamma m} \frac{|e^{-\pe}\alpha m - e^{-\pe}(\alpha+\beta\frac{d-1}{m})(m-d+1)|}{(e^{-\pe}\alpha(1-\gamma)m)^2} \right|\\
		&\leq \frac{e^{\pe}}{\alpha(1-\gamma)}\sqrt{8\log(\frac{2}{\delta})\frac{\gamma}{m}} + \left|\sum\limits_{d=1}^{\gamma m} \frac{\max(\alpha, \beta)d }{e^{-\pe}\alpha^2(1-\gamma)^2m^2} \right|\\
		&\leq  \frac{e^{\pe}}{\alpha(1-\gamma)}\sqrt{8\log(\frac{2}{\delta})\frac{\gamma}{m}} +  \frac{e^{\pe}\max(\alpha,\beta)\gamma^2 }{\alpha^2(1-\gamma)^2} \\
		&\leq \frac{2e^{\pe}(\alpha+\beta)}{\alpha^2(1-\gamma)^2}\sqrt{8\log(\frac{2}{\delta})}m^{-2/3}
	\end{align*}
	\newcommand{\tgammambound}{\frac{2e^{\pe}(\alpha+\beta)}{\alpha^2(1-\gamma)^2}\sqrt{8\log(\frac{2}{\delta})}m^{-2/3}}
	
	Denote the above bound by $\cB_\alpha$. Plug this and $\pe=0$ into the $|\hat
	\alpha-\alpha|$ expression above we have with probability $1-\delta$:
	\begin{align*}
		|\hat \alpha-\alpha| &\leq \gamma \frac{\cB_\alpha}{\left(\gamma m\bbE[\tau_1] - \cB_\alpha\right)^2}\\
		&\leq \gamma \frac{4\cB_\alpha}{\gamma^2 m^2\bbE[\tau_1]^2 }\\
		&=\frac{8(\alpha+\beta)}{(1-\gamma)^2}\sqrt{8\log(\frac{2}{\delta})}m^{-1/3}
    \end{align*}
    where in the second inequality we used the fact that $m^{1/3} \geq
    \frac{16(\alpha+\beta)}{\alpha}\sqrt{8\log(\frac{2}{\delta})} \implies
    \cB_\alpha \leq \frac{1}{2}\gamma m\bbE[\tau_1]$.

\end{proof}

\begin{restatable}{lemma}{betahatbound}
    \label{lem:betahatbound}
    Assuming that $m^{1/3} \geq
    64\frac{(\alpha+\beta)^2}{\alpha^2}\sqrt{8\log(\frac{2}{\delta})}$, then for any $i=1,\ldots, K$, 
    with probability $1-\delta$, $|\beta-\hat \beta_i|
    \leq B_i$.
\end{restatable}
\begin{proof}
	Note that $\bbE[I_{\gim+1}] = \bbE[\tau_{\gamma_i m+1} - \tau_{\gamma_i m}] = \frac{1}{\lambda(p_0, \gamma_i)}=\frac{1}{e^{-p_0}(\alpha+\beta\gamma_i)(1-\gamma_i)m}$. From \eqref{eq:betahat} we can bound the estimation error of $\hat \beta$ as follows. We have
	\begin{align*}
	\hat \alpha +\hat \beta_i\gamma_i & = \frac{\gamma m}{e^{-\pe}(1-\gamma_i)m (\tau_{(\gamma_i+\gamma)m} - \tau_{\gamma_i m})},\\
	\alpha + \beta\gamma_i & = \frac{\gamma m}{e^{-\pe}(1-\gamma_i)m \bbE[\tau_{\gamma_i m+1} - \tau_{\gamma_i m}] \gamma m}.
	\end{align*}
	Therefore,
	$$|\beta -\hat \beta_i| \leq  \frac{|\hat \alpha - \alpha|}{\gamma_i} + \frac{\gamma m}{\gamma_ie^{-\pe}(1-\gamma_i)m}\left|\frac{1}{\tau_{(\gamma_i+\gamma )m} - \tau_{\gamma_i m}} - \frac{1}{\gamma m \bbE[\tau_{\gamma_i m+1} -\tau_{\gamma_i}]}\right|.$$
    Similar to the estimation bound of $\hat\alpha$, the main step is to bound the arrival times.
    Note that $\lambda(p_0, \frac{d-1}{m}) \geq 
    e^{-p_0}(\alpha+\beta\gamma_i)(1-\gamma_i-\gamma)m \geq \underline{\lambda} \coloneqq
    e^{-p_0}(\alpha+\beta\gamma_i)(1/3-\gamma)m$ for $d\in \{\gamma_im+1,\ldots,(\gamma_i+\gamma)m\}$, where we
    used the fact that by the construction of Algorithm~\ref{alg:main},
    $\gamma_i \leq 2/3$ for all $i=1,\ldots, K$.
    Using Lemma~\ref{lem: general exponential arrival time martingale
    concentration bound} we have:
    \begin{align*}
        \bbP\left( \left|\tau_{\gamma m} - \sum\limits_{d=1}^{\gamma m} \bbE[I_d]\right| \geq \epsilon\right)
		\leq &2exp(-\frac{\epsilon^2\underline{\lambda}^2}{8\gamma m})\\
	 \text{ so that } \left|\tau_{\gamma m} - \sum\limits_{d=1}^{\gamma m} \frac{1}{\lambda(p_0, d)} \right| 
	\leq &\frac{e^{p_0}}{(\alpha+\beta\gamma_i)(1/3-\gamma)}\sqrt{8log(\frac{2}{\delta})\frac{\gamma}{m}} \quad \text{with probability } 1-\delta
    \end{align*}
    where we set $\epsilon = \sqrt{\frac{8\log(\frac{2}{\delta})\gamma m}{\underline{\lambda}^2}}$.
    Since $m^{1/3}\geq \sqrt{2\log(\frac{2}{\delta})} \implies
    \epsilon \leq \frac{2\gamma m}{\underline{\lambda}}$, the condition for
    Lemma~\ref{lem: general exponential arrival time martingale concentration
    bound} is satisfied. To bound the estimation error of $\beta$:

    \begin{align*}
		     &\left|\left(\tau_{(\gamma_i+\gamma) m} - \tau_{\gamma_i m}\right) - \gamma m\bbE[\tau_{\gamma_i m+1}-\tau_{\gamma_i m}]\right| 
		=     \left|\sum\limits_{d=\gamma_i m +1}^{(\gamma_i+\gamma) m} I_d - \frac{\gamma m}{\lambda(\pe, \gamma_i)}\right|\\
        \leq &\left|\sum\limits_{d=\gamma_i m +1}^{(\gamma_i+\gamma) m} \left(I_d - \frac{1}{\lambda(\pe, \frac{d-1}{m})}\right) \right| + \left|\sum\limits_{d=\gamma_i m +1}^{(\gamma_i+\gamma) m}\frac{1}{\lambda(\pe, \frac{d-1}{m})} - \frac{\gamma m}{\lambda(\pe, \gamma_i)}\right|\\
        \text{(w.p. $1-\delta$)}
		\leq &\frac{e^{\pe}}{(1/3-\gamma)(\alpha+\beta\gamma_i)}\sqrt{8\log(\frac{2}{\delta})\frac{\gamma}{m}} + \left|\sum\limits_{d=\gamma_im+1}^{\gamma_im+\gamma m} (\frac{1}{\lambda(\pe, d)} -\frac{1}{\lambda(\pe, \gamma_i)})\right|\\
		\leq &\frac{e^{\pe}}{(1/3-\gamma)(\alpha+\beta\gamma_i)}\sqrt{8\log(\frac{2}{\delta})\frac{\gamma}{m}}\\
		&		      + \left|\sum\limits_{d=\gamma_im+1}^{\gamma_im+\gamma m }
		      \frac{|e^{-\pe}(\alpha+\beta\frac{d-1}{m})(m-d+1) -
		      e^{-\pe}(\alpha+\beta\gamma_i)(1-\gamma_i)m|}{(e^{-\pe}(\alpha+\beta\gamma_i)(1-\gamma_i-\gamma)m)^2}\right|\\
		\leq &\frac{e^{\pe}}{(1/3-\gamma)(\alpha+\beta\gamma_i)}\sqrt{8\log(\frac{2}{\delta})\frac{\gamma}{m}} + \left|\sum\limits_{d=\gamma_im+1}^{\gamma_im+\gamma m } (\frac{(\alpha+\beta)\gamma }{e^{-\pe}(\alpha+\beta\gamma_i)^2(1-\gamma_i-\gamma)^2m} )\right|\\
		\leq &\frac{e^{\pe}}{(1/3-\gamma)(\alpha+\beta\gamma_i)}\sqrt{8\log(\frac{2}{\delta})\frac{\gamma}{m}} +  \frac{e^{\pe}(\alpha+\beta)\gamma^2 }{(\alpha+\beta\gamma_i)^2(1/3-\gamma)^2} \\
		\leq &\frac{2e^{\pe}(\alpha+\beta)}{(\alpha+\beta\gamma_i)^2(1/3-\gamma)^2}\sqrt{8\log(\frac{2}{\delta}) }m^{-2/3}
	\end{align*}

	Let $\cB_\beta$ denote this bound. Plug this result back into the bound on $|\beta-\hat \beta_i|$:
	\begin{align*}
	    |\beta -\hat \beta_i| &\leq \frac{|\hat \alpha - \alpha|}{\gamma_i} +
	    \frac{\gamma
	    m}{\gamma_ie^{-\pe}(1-\gamma_i)m}\left(\frac{1}{\tau_{(\gamma_i+\gamma
	    )m} - \tau_{\gamma_i m}} - \frac{1}{\gamma m \bbE[\tau_{\gamma_i m+1}
	    -\tau_{\gamma_i}]}\right)\\
	    &\leq \frac{|\hat \alpha - \alpha|}{\gamma_i} + \frac{e^{\pe}\gamma
	    }{\gamma_i(1-\gamma_i)}\left(
            \frac{\cB_\beta}
                 {(\gamma m \bbE[I_{\gamma_i m+1}] - \cB_\beta)^2}
        \right)\\
        \text{(*)}
	    &\leq \frac{|\hat \alpha - \alpha|}{\gamma_i} + \frac{e^{\pe}\gamma
	    }{\gamma_i(1-\gamma_i)}\left(
            \frac{4\cB_\beta}
                 {\gamma^2 m^2 \bbE[I_{\gamma_i m+1}]^2}
        \right)\\
        &\leq \frac{|\hat \alpha - \alpha|}{\gamma_i} + 
		\frac{8(\alpha+\beta)(1-\gamma_i)}{\gamma_i(1/3-\gamma)^2}\sqrt{8\log(\frac{2}{\delta})}m^{-1/3}\\
        &\leq
        \frac{16(\alpha+\beta)(1-\gamma_i)}{\gamma_i(1/3-\gamma)^2}\sqrt{8\log(\frac{2}{\delta})}m^{-1/3}
    \end{align*}
    where for the (*) step we used the fact that $m^{1/3} \geq
    64\frac{(\alpha+\beta)^2}{\alpha^2}\sqrt{8\log(\frac{2}{\delta})}
    \implies \cB_\beta\leq \frac{1}{2}\gamma m \bbE[I_{\gamma_i m +1}]$.

\end{proof}

\lcbonoptprice*
\begin{proof}
    Clearly, $\hat\alpha - A \leq \alpha$, $\hat \beta_k - B_k \leq \beta$, and $\hat \beta_k + B_k \geq \beta$. 
    Then using Lemma~\ref{lem: opt price lipschitz bound} we have
    \begin{align*}
        |\pcurve(x, \hat\alpha, \hat\beta_k) - \pcurve(x, \alpha, \beta)|
        \leq &\max\limits_{\alpha' \in [\alpha, \hat\alpha] \text{ or } \newshipra{\alpha'}\in[\hat\alpha, \alpha]} \left\{
             \left|\frac{\partial \pcurve(x, \alpha, \beta)}{\partial
             \alpha}\right|_{\alpha=\alpha'}\right\}|\alpha-\hat\alpha|\\
             &+\max\limits_{\beta' \in [\beta, \hat\beta] \text{ or } \newshipra{\beta'}\in[\hat\beta, \beta]} \left\{
             \left|\frac{\partial \pcurve(x, \alpha, \beta)}{\partial
             \beta}\right|_{\beta=\beta'}\right\}|\beta-\hat\beta_k|\\
        \leq &\alphaL A + \betaL{i} B_k
    \end{align*}
\end{proof}

\subsection{Step 2: Proof of Lemma \ref{lem:epochbound}}
\label{sec: per epoch and overall regret bounds}

Since the prices that we offer in the stochastic Bass model is based
on a discretized version of the continuous price curve in the deterministic
Bass model, we first need to prove a result that says that this discretization does not
introduce a lot of error. Lemma~\ref{lem: discretization is good enough} below
shows that it only introduces a logarithmic (in $m$) amount of error,  for any fixed $\alpha, \beta, T$.

\begin{lemma}
    \label{lem: discretization is good enough}
    For any fixed $T, \alpha, \beta$ and $n =1,\ldots,m\Xopt_T$,
    \removedsy{
    }
    $$\left|\sum\limits_{d=1}^{n} \pcurve\left(\frac{d-1}{m}, \alpha,
    \beta\right) - m\int_{0}^{n/m} \pcurve(x, \alpha, \beta) dx\right| \leq
    \frac{n}{2m}\frac{\beta}{\alpha}+\frac{1}{2}\log\left(m\right) +2\pcurve_{max}$$
    where $\pcurve_{max}$ denotes an upper bound on the optimal prices $\pcurve(x, \alpha, \beta)$ for all $x$.
\end{lemma}
\begin{proof}
    Using \eqref{eq: opt_p_model_in_opt_trajectory}, 
    $$\left|\frac{\partial \pcurve(x, \alpha, \beta)}{\partial x} \right|\leq \frac{\beta}{\alpha}+\frac{1}{1-x}.$$
In below we abbreviate $\pcurve(x,\alpha, \beta)$ as $\pcurve_x$.
    \removedsy{
    }
    \begin{align*}
        &\left|\sum\limits_{d=1}^{n} \pcurve\left(\frac{d-1}{m}, \alpha, \beta\right) -
            m\int_{0}^{n/m} \pcurve(x,\alpha, \beta) dx\right| \\
            \leq & 
        \sum\limits_{d=1}^{n-2} \left|\pcurve_{\frac{d-1}{m}} -
            m\int_{\frac{d-1}{m}}^{d/m} \pcurve_x dx\right| + 2\pcurve_{max}\\
        \removedsy{\text{\sacomment{remove this step?}}\leq &\sum\limits_{d=1}^{n-2}\left[
        m\int_{\frac{d-1}{m}}^{d/m}\left(\frac{\beta}{\alpha} +
        \frac{1}{1-x}\right)(x-\frac{d-1}{m})dx \right] + 2\pcurve_{max}\\}
        \leq &\sum\limits_{d=1}^{n-2}\left[
        m\int_{\frac{d-1}{m}}^{d/m}\left(\frac{\beta}{\alpha} +
        \frac{1}{1-d/m}\right)(x-\frac{d-1}{m})dx \right] + 2\pcurve_{max}\\
        =& \frac{n}{2m}\frac{\beta}{\alpha} +\sum\limits_{d=1}^{n-2} \frac{1}{2(1-d/m)m}+ 2\pcurve_{max}\\
        \leq& \frac{n}{2m}\frac{\beta}{\alpha} +\frac{1}{2m} \int_{1}^{n-1}\frac{1}{(1-g/m)}dg+ 2\pcurve_{max}\\
        \leq& \frac{n}{2m}\frac{\beta}{\alpha}+\frac{1}{2}\log\left(\frac{1}{1-\frac{n-1}{m}}\right)+ 2\pcurve_{max}\\
        \leq& \frac{n}{2m}\frac{\beta}{\alpha}+\frac{1}{2}\log\left(m\right)+ 2\pcurve_{max}
    \end{align*}
    The first inequality follows because we know $\pcurve(x, \alpha, \beta)$ is bounded below and above by $0$ and $\pcurve_{max}$. Therefore the difference between $\pcurve$ evaluated at two different $x$ values is at most $\pcurve_{max}$.
\end{proof}
\begin{corollary}
\label{cor: discretization is good enough}
    For any fixed $T, \alpha, \beta$, 
    \removedsy{assume that $m-\lfloor m\Xopt_T(x_0)\rfloor \geq 2$, then}
    $$\left|\sum\limits_{d=1}^{\lfloor m\Xopt_T\rfloor} \pcurve\left(\frac{d-1}{m}, \alpha,
    \beta\right) - m\int_{0}^{\Xopt_T} \pcurve(x, \alpha, \beta) dx\right| \leq
    \frac{\beta}{2\alpha}+\frac{1}{2}\log\left(m\right)+3\pcurve_{max}$$
\end{corollary}
\begin{proof}
    \removedsy{
    Let $n= \lfloor m\Xopt_T\rfloor$, note that 
    \begin{align*}
        \log\left(\frac{1}{1-\frac{n+1}{m}}\right) = &\log\left(\frac{m}{m- \lfloor m\Xopt_T(x_0)\rfloor -1}\right)\\
        \leq &\log(m)
    \end{align*}
    Since rounding $m\Xopt_T$ introduces at most $\pcurve_{max}$ difference in revenue, the result then immediately follows from Lemma~\ref{lem: discretization is good enough}. }
    Let $n= \lfloor m\Xopt_T(x_0)\rfloor - d_0$. Since rounding $m\Xopt_T$ introduces at most at additional $\pcurve_{max}$ difference in revenue, the result then immediately follows from Lemma~\ref{lem: discretization is good enough}.
\end{proof}

 For the lemma below, we define 
    $$\detV_i(T) \coloneqq m\int_{\gamma_i}^{2\gamma_i \wedge \Xopt_T} \pcurve(x, \alpha,\beta)\,dx$$
    $$\rev_i    \coloneqq  \sum_{d=\lfloor\gamma_i m\rfloor+1}^{\lfloor2\gamma_i m\rfloor \wedge \lfloor\Xopt_T m\rfloor} p_d$$

\epochbound*
\begin{proof}
	Let $A, B_i$ be the bound on the estimation error stated in
	Lemma~\ref{lem:alphahatbound}, Lemma~\ref{lem:betahatbound}. Recall that
	first $\gamma m$ customers in every epoch are offered an exploration
	price $p_0=0$. Let $p_d$ be the price paid by customer $d$ as
	specified in Algorithm~\ref{alg:main} and \eqref{eq: lcb on opt price}.
	Recall also that $\pcurve(x, \alpha, \beta)$ is the optimal price curve
	as specified in \eqref{eq: opt_p_model_in_opt_trajectory}. We use the
	short hand notations $\pcurve_x\coloneqq \pcurve(x, \alpha, \beta),
	\pcurve_d\coloneqq \pcurve(\frac{d-1}{m}, \alpha, \beta), \pcurve_{max} \coloneqq
	\max\limits_{x\in[0,1)} \pcurve_x$ in the following calculations.
    \begin{align*}
		\detV_i-\rev_i &= m\int_{\gamma_i}^{2\gamma_i \wedge \Xopt_T}\pcurve_x\,dx - \sum\limits_{d=\gamma_i m+1}^{2\gamma_im \wedge \lfloor\Xopt_T m\rfloor}p_d\\ 
        &= \sum\limits_{d=\gamma_i m+1}^{2\gamma_im \wedge \lfloor \Xopt_T m\rfloor}
        \left[p^*_d-p_d\right] + \left|\sum\limits_{d=\gamma_i m +1}^{2\gamma_i m\wedge \lfloor\Xopt_T m\rfloor} \pcurve_d -
        m\int_{\gamma_i}^{2\gamma_i \wedge \Xopt_T} \pcurve_x dx\right| \\
        \text{(Corollary~\ref{cor: discretization is good enough})}&\leq \sum\limits_{d=\gamma_i m+1}^{2\gamma_im \wedge \lfloor\Xopt_T m\rfloor}
        \left[\pcurve_d-p_d\right] + \frac{\beta}{2\alpha}+\frac{1}{2}\log\left(m\right)+3\pcurve_{max}\\
        &\leq (\gamma m+3) \pcurve_{max} + \sum\limits_{d=(\gamma_i+\gamma) m+1}^{2\gamma_im \wedge \lfloor \Xopt_T m\rfloor}
        \left[p^*_d-p_d\right] + \frac{\beta}{2\alpha}+\frac{1}{2}\log\left(m\right)\\
        \text{(Lemma~\ref{lem: lcb on opt price},
        Lemma~\ref{lem:alphahatbound}, Lemma~\ref{lem:betahatbound})}&\leq
        (\gamma m+3) \pcurve_{max} + 2\gamma_i m(\alphaL A+\betaL{i} B_i) +
        \frac{\beta}{2\alpha}+\frac{1}{2}\log\left(m\right) \quad w.p. 1-2\delta\\
        \text{(Lemma~\ref{lem:priceupperbound})}&\leq
        O\left(m^{2/3}\log\left((\alpha+\beta)T\right)\right) + 2\gamma_i m(\alphaL A+\betaL{i} B_i)
    \end{align*}
    where $\alphaL, \betaL{i}$ are defined in \eqref{eq: lcb on opt
    price}. 
    
    
    The second to last step follows because we know that
    $|\lcb{\pcurve_i}(\frac{d-1}{m})-\pcurve_d| \leq 2(\alphaL A+\betaL{i}
    B_i)$ using the definition of $\lcb{\pcurve_i}$ in \eqref{eq: lcb on opt
    price} and Lemma~\ref{lem: lcb on opt price}, and that from
    Lemma~\ref{lem:alphahatbound} and Lemma~\ref{lem:betahatbound} we know
    that the error bounds $A$, $B_i$ hold with probability $1-2\delta$.

    Assuming that $m^{1/3} \geq \frac{64\phi}{\alpha}
    \sqrt{8\log(\frac{2}{\delta})}$, and $\gamma_im^{1/3}\geq
    \frac{640\phi}{\beta} \sqrt{8\log(\frac{2}{\delta})}$, one can check this
    implies that $A\leq \frac{\alpha}{4}$ and $B_i\leq \frac{\beta}{4}$,
    which in turn implies that $\hat\alpha - A \geq \frac{\alpha}{2}$ and
    $\hat\beta-B_i\geq \frac{\beta}{2}$. Applying these bounds to \eqref{eq: lcb on opt price} we have:
    $$\alphaL \leq \frac{4}{\alpha}+\frac{6\beta}{\alpha^2}\qquad \betaL{i}
    \leq 6\left(\frac{1}{\alpha}+\frac{1}{\beta}\right)$$

    Plug in the expressions of $A$ and $B_i$, as well as the above bounds on
    $\alphaL, \betaL{i}$, we have with probability $1-\delta$,
    $$\detV_i(T) - \rev_i \leq O\left(\epochboundorder\right)$$

    When 
    $m^{1/3} \leq \frac{64\phi}{\alpha}
    \sqrt{8\log(\frac{2}{\delta})}$, or $\gamma_im^{1/3}\leq
    \frac{640\phi}{\beta} \sqrt{8\log(\frac{2}{\delta})}$, the regret can be
    trivially bounded by $\pcurve_{max} \gamma_i m\leq
    O\left(\log((\alpha+\beta)T)\left(\frac{\phi}{\alpha}
    \sqrt{\log(\frac{1}{\delta})}\right)^3 + 
    \log((\alpha+\beta)T)\frac{\phi}{\beta} \sqrt{\log(\frac{1}{\delta})}m^{2/3}\right)$,
    where we used Lemma~\ref{lem:priceupperbound} to bound $\pcurve_{max}$.
    Since the first component is a constant with respect to $m$ and the
    second is no larger than the expression from before, we are done.

\end{proof}

\subsection{Step 3: Proof of Lemma \ref{lem:finalstochXclosetooptimal} and Lemma \ref{lem:priceupperbound}}

\finalstochXclosetooptimal*
\begin{proof}
    Let $p_d$ be the prices that the algorithm offers for customer $d$. Since	Algorithm~\ref{alg:main} offers either $p_0=0$ or the lower confidence bound price defined in \eqref{eq: lcb on opt price}, we know from
	Corollary~\ref{cor: lcb on opt price}, as well as Lemma~\ref{lem:alphahatbound}, Lemma~\ref{lem:betahatbound}, that with probability $1-\delta K$,  $p_d\leq
	p^*(\frac{d-1}{m}, \alpha, \beta) \forall d \leq d_T\wedge m\Xopt_T$, where $K$ is the total number of epochs. 
	This means that $\lambda(p_d, \frac{d-1}{m})
    \geq \lambda(\pcurve(\frac{d-1}{m},\alpha, \beta), \frac{d-1}{m}) = \frac{1}{e}(\alpha+\beta X^*_T)(1-X^*_T)m$.
    Let $\underline{\lambda} \coloneqq \frac{1}{e}(\alpha+\beta
    \Xopt_T)(1-\Xopt_T)m$, which
    by \eqref{eq: intro XT_model_quadratic_sol1} is equal to $ \frac{\Xopt_T m}{T}$.
        
    Set $n=\lfloor m\Xopt_T - \sqrt{8m\Xopt_T\log(\frac{2}{\delta})}\rfloor,
    \underline{\lambda}=\frac{m\Xopt_T}{T}, \epsilon =
    \sqrt{\frac{8n}{\underline{\lambda}^2}log(\frac{2}{\delta})}$, then by
    Lemma~\ref{lem: general exponential arrival time martingale concentration
    bound}, we have with probability $1-\delta$:
    $$\left|\tau_{n} -
    \left(1-\sqrt{\frac{8\log(\frac{2}{\delta})}{m\Xopt_T}}\right)T\right|
    \leq T\sqrt{\frac{8}{\Xopt_T m }log(\frac{2}{\delta})}$$
    This means that with probability $1-\delta (K+1)$,$\tau_n \leq T$, which means that
    the total number of adoptions observed in the stochastic Bass model is at
    least $n$. The result follows by observing that there can be at most $\log(m)$ epochs. 
\end{proof}


\priceupperbound*
\begin{proof}
Here we use the expanded notation of $\Xopt_T(x, \alpha, \beta)$ introduced in Appendix~\ref{sec: det opt price properties}.
Using \eqref{eq: opt_p_model_in_opt_trajectory}, we have for any $x\leq \Xopt_T(0, \alpha, \beta)$:
\begin{align*}
    \pcurve(x, \alpha, \beta)\leq& 1+\log\left(\frac{1-x}{1-\Xopt_T(0, \alpha, \beta)}\right)\\
    \leq& 1+\log\left(\frac{1}{1-\Xopt_T(0, \alpha, \beta)}\right)\\
    \leq&\log\left(e+(\alpha+\beta) T\right)
\end{align*}
The last step follows from the following upper bound on $\Xopt_T$:
$$
\Xopt_T(0, \alpha, \beta) \leq \Xopt_T(0, \alpha+\beta, 0)=\frac{(\alpha+\beta) T}{(\alpha+\beta) T+e}
$$
The first inequality is easy to see from \eqref{eq: intro XT_model_quadratic_sol1}: by replacing $\alpha+\beta \Xopt_T$ with $\alpha+\beta$, we can see that $\Xopt_T/(1-\Xopt_T)$ increases, and since this quantity is strictly monotone in $\Xopt_T$, $\Xopt_T$ must be larger. And the last equality follows from solving \eqref{eq: intro XT_model_quadratic_sol1} after replacing $\alpha+\beta \Xopt_T$ with $\alpha+\beta$.
\end{proof}

\subsection{Step 4:  Putting it all together for proof of Theorem \ref{thm:main}}
\begin{proof}[Proof of Theorem \ref{thm:main}]
    \begin{align*}
        \preg &= \sum_{i=1}^K \left[\detV_i - \rev_i\right] + \sum\limits_{d=d_T+1}^{m\Xopt_T}p_d\\
        &\leq \log(m)O\left(\epochboundorder\right)  \quad \text{w.p. $1-(K+1)\delta $}\\
        &\quad + O\left(\log\left((\alpha+\beta)T\right)\sqrt{m}\right)\\
        &=    O\left(\regretboundorder\right)
    \end{align*}
    where the per epoch regrets are bounded using Lemma~\ref{lem:epochbound} and the ``uncaptured revenue'' term is bounded using Lemma~\ref{lem:finalstochXclosetooptimal} and \ref{lem:priceupperbound}.

    The inequality holds with probability $1-(K+1)\delta$, and 
    $K$ is the total number of epochs, which is bounded by $\log(m)$
    since the epoch length is defined with respect to the number of customers
    and increases geometrically. 
\end{proof}
\section{Lower Bound Proofs}
\label{sec: lower bound proofs}
\subsection{Step 1: missing lemmas and proofs}

\instantaneousimpactofsuboptimalprice*
\begin{proof}
    Let $\popt$ denote $\popt(x, T')$ for the remainder of this proof.
    Using \eqref{eq: deterministic value function bellman equation}, we can rewrite the left hand side of the lemma:
    \begin{align}
        & A^{det}(x, T', p) \nonumber\\
        = &\lim_{\delta \rightarrow 0} \frac{\detV(x, T')-p\lambda(p, x)\delta - \detV(x+\lambda(p, x)\delta/m, T'-\delta)   }{\delta}\nonumber\\
         = &\lim_{\delta \rightarrow 0}\frac{\popt \lambda(\popt,x)\delta + \detV(x+\lambda(\popt,x) \delta/m, T'-\delta)-p\lambda(p, x)\delta - \detV(x+\lambda(p, x)\delta/m, T'-\delta)   }{\delta}\nonumber\\
        = &\detG(x, T', \popt) - \detG(x, T', p)\label{eq: instantaneous value loss as G}
    \end{align}
    where we define 
    \begin{align}
    \detG(x, T', p) 
        &=\lim\limits_{\delta\rightarrow 0}
        p\lambda(p, x) + \frac{\detV(x+\lambda(p, x)\delta/m, T'-\delta) -
        \detV(x, T'-\delta)}{\delta}\nonumber\\
        &= p\lambda(p, x) + \frac{\lambda(p,x)}{m} \frac{\partial \detV(x, T')}{\partial x} \label{eq: det G function}
    \end{align}
        
    The distribution of limits is valid since both limits exist. The new quantity $\detG(x, T', p)$ will help us quantify the \emph{instantaneous} impact, or the (dis)Advantage, of offering a suboptimal price. 

    From above, note that $\pi^*(x,T') := \argmin_p A^{det}(x,T',p) = \arg \max_p G^{det}(x,T', p)$. We derived the expression for $\pi^*(x,T')$ earlier in the proof of Lemma \ref{lem: concavity of deterministic value function} (equation \eqref{eq: opt det price wrt value function}) as 
    $$\popt(x,T') = 1-\frac{1}{m}\frac{\partial \detV(x, T')}{\partial x}.$$


    We can now bound \eqref{eq: instantaneous value loss as G} using the derivative of $\detG$:
    \begin{align*}
    \frac{\partial \detG(x, T', p)}{\partial p}&= \lambda(p, x) - p\lambda(p, x) -\frac{\lambda(p,x)}{m} \frac{\partial \detV(x, T')}{\partial x} \\ 
    &=(\popt-p)\lambda(p, x)\\
    \detG(x, T', \popt) - \detG(x, T', p)
    &= \int_{p}^{\popt} (\popt-\rho) \lambda(\rho, x)d\rho\\ 
    &= \lambda(p, x)(\popt-p) + \lambda(\popt, x) - \lambda(p, x)\\ 
    &= \lambda(p, x)(e^{p-\popt} - 1-(p-\popt))\\ 
    &\geq \lambda(p, x) \min\left(\frac{1}{10}, \frac{1}{4}(p-\popt)^2\right)
    \end{align*}
    To see the last step, consider $f(x) = e^x-1-x$ which is a convex
    function minimized at $x=0$. For $x\geq -\frac{1}{2}$, it's easy to show
    that $f''(x)\geq e^{-1/2} \geq \frac{1}{2}$. So by standard strong
    convexity argument $f(x)\geq \frac{1}{4} x^2$. For $x\leq -\frac{1}{2}$
    we have $f(x)\geq f(-\frac{1}{2})\geq \frac{1}{10}$.
\end{proof}

Before we translate the above result into a regret bound in terms of cumulative pricing error, we explain the proof idea with some more details. 
\removedsy{
}
Given any arbitrary pricing algorithm, let \begin{center}
    $[(p_1, I_1), \ldots, (p_{n}, I_{n})]$
\end{center} be the first $n$ observations (tuples of price $p_d$ and inter-arrival times $I_d$ between customer $d-1$ and $d$) in the stochastic Bass
model, on following the algorithm's pricing policy. Here we assume that the algorithm is allowed to continue running even after the planning horizon $T$ has passed. If the algorithm is undefined after time $T$, we assume that it offers $0$ price.
We use these observations to define a continuous price trajectory $p_x, 0\le x \le n/m$ as follows: set $p_x=p_d,  \forall\, x\in[\frac{d-1}{m}, \frac{d}{m}), d=1,\ldots, n$. We call $p_x$ the induced price trajectory of $p_d$. Let $t^{det}_{n/m}$ denote the time when the adoption level
hits $x$ in the \emph{deterministic} Bass model following this induced pricing trajectory $p_{x}$.
In other words, $t^{det}_{x} = \int_0^{x}\frac{m}{\lambda(x', p_{x'})}dx'$.
Note that $t^{det}_{n/m}$ is a
stochastic quantity because it depends on stochastic trajectory $p_x$, which in turn depends on the prices $p_d$ offered to the first $n$ customers in the stochastic model. 

Recall that $\tau_{n}=\sum\limits_{d=1}^{n}I_d$ denotes  
the arrival time of customer $n$ in the  stochastic model. First we show that the total time for $n$ customer arrivals in the deterministic vs. stochastic model (i.e., $t^{det}_{n/m}$ vs. $\tau_n$) under the two price trajectories ($p_x$ vs. $p_d$) is roughly the same.

\begin{restatable}{lemma}{detarrivalvsstocharrival}
    \label{lem: detarrival vs stocharrival}
    Given $n \le m,\delta \in (0,1)$ such that $n \geq 2\log(\frac{2}{\delta})$. Then for any algorithm satisfying Assumption \ref{assum: const price in between arrivals}, with probability at least
    $1-\delta$,
    $$|t^{det}_{n/m} - \tau_{n}| \leq \frac{e^{p_{max}}}{\alpha (m-n)}\sqrt{8n\log(\frac{2}{\delta})} + \frac{e^{p_{max}}(\alpha+\beta)n}{2\alpha^2(m-n-1)^2m}$$
    where $p_{max}$ is an upper bound on the prices $p_d, 1\le d \le d_T$ offered by  algorithm.
\end{restatable}
\begin{proof}
    Since algorithm's prices are fixed between arrival of two customers (by Assumption \ref{assum: const price in between arrivals}), we have that inter-arrival times follow an exponential distribution. That is, for any $d$, given $p_{d+1}$, the price paid by $(d+1)^{th}$ customer, $\tau_{d+1}-\tau_d$ follows the exponential distribution $   Exp(\lambda(p_{d+1}, \frac{d}{m}))$, where $\lambda(p,
    x)=e^{-p}(\alpha+\beta x)(1-x)m$. Note that $ \lambda(p_{d+1},\frac{d}{m})\ge \underline{\lambda} \coloneqq
    e^{-p_{max}}\alpha(m-n)$ for all $d\leq n$.
    Set $\epsilon = \sqrt{\frac{8n }{\underline{\lambda}^2}\log(\frac{2}{\delta})}$, then
    Lemma~\ref{lem: general exponential arrival time martingale concentration
    bound} in Appendix~\ref{sec: concentration result} provides that with probability $1-\delta$:
    $$|\tau_n - \sum\limits_{d=1}^{n}\bbE[\tau_d-\tau_{d-1}|\cF_{d-1}]| \leq \sqrt{\frac{8n}{\underline{\lambda}^2}\log(\frac{2}{\delta})}$$
    Note that the condition on $\epsilon$ in Lemma~\ref{lem: general
    exponential arrival time martingale concentration bound} is 
    satisfied since
    $\sqrt{\frac{8n}{\underline{\lambda}^2}\log(\frac{2}{\delta})}
    \leq \frac{2n}{\underline{\lambda}} \iff n\geq
    2\log(\frac{2}{\delta})$.
    
    On the other hand, for any $d\leq n$,
    $t^{det}_{\frac{d+1}{m}} - t^{det}_{\frac{d}{m}} = m
    \int_{d/m}^{\frac{d+1}{m}} \frac{1}{\lambda(p_{d+1}, x)} dx$. 
    It's easy to check that $|\frac{\partial}{\partial x}\frac{1}{\lambda(p, x)}| = |\frac{1}{m}\frac{e^p[\beta(1-x)-(\alpha+\beta x)]}{(\alpha+\beta x)^2(1-x)^2}| \leq \frac{e^p(\alpha+\beta)}{\alpha^2(1-x)^2m}$.
    So 
    \begin{align*}
	\left|t^{det}_{\frac{d+1}{m}} - t^{det}_{d/m} - \bbE[\tau_{d+1}-\tau_d | \cF_{d}]\right| =&
	\left|m \int_{d/m}^{\frac{d+1}{m}} \frac{1}{\lambda(p_{d+1}, x)} dx -
    \frac{1}{\lambda(p_{d+1}, d/m)}\right|\\
	\leq &\left|m \int_{d/m}^{\frac{d+1}{m}} \frac{1}{\lambda(p_{d+1}, \frac{d}{m})} + \frac{e^{p_{d+1}}(\alpha+\beta)}{\alpha^2(1-\frac{d+1}{m})^2m}(x-\frac{d}{m})dx -
    \frac{1}{\lambda(p_{d+1}, d/m)}\right|\\
	\leq &\left|\frac{e^{p_{d+1}}(\alpha+\beta)}{\alpha^2(1-\frac{d+1}{m})^2} \int_{d/m}^{\frac{d+1}{m}}
	(x-\frac{d}{m})dx\right|\\
	\leq& \frac{e^{p_{d+1}}(\alpha+\beta)}{2\alpha^2(1-\frac{d+1}{m})^2m^2}
    \end{align*}
    In the second to last step we canceled the first and third term from the previous step.
    Combining this with above, we have with probability $1-\delta$
    $$|t^{det}_{n/m} - \tau_{n}|\leq \frac{e^{p_{max}}}{\alpha
    (m-n)}\sqrt{8n \log(\frac{2}{\delta})} + 
    \frac{e^{p_{max}}(\alpha+\beta)n}{2\alpha^2(m-n-1)^2m}$$
     
\end{proof}
Using $t^{det}_x$ and the induced pricing trajectory $p_x$ as defined right before Lemma~\ref{lem: detarrival vs stocharrival}, we can now obtain the following result.
\begin{restatable}{lemma}{lowerboundpseudoregretpricedifference}
\label{lem: lower bound pseudo regret price difference}
Fix any $\alpha, \beta$. Then for any $m, T$  such that $m\Xopt_T\geq n:=\left\lfloor \left(\frac{\alpha}{\alpha+\beta}\right)^{4/3} m^{2/3}\right\rfloor$, and any algorithm satisfying Assumption \ref{assum: const price in between arrivals} and \ref{assum:price upper bound},
$$\bbE\left[\preg(T)\right] \ge \bbE\left[m\int_{0}^{n/m}
    \min\left(\frac{1}{4}(\popt_x-p_x)^2,
    \frac{1}{10}\right) dx\right]  -O\left({\frac{(\alpha+\beta)^{1/3}T e^C}{\alpha^{1/3}}}m^{1/3}\sqrt{\log(m)}\right)$$
where $\popt_x:=\popt(x, \alpha, \beta, T-t^{det}_x)$ denotes the deterministic optimal price at adoption level $x$, $p_x:=p_d$ denotes the price offered by the algorithm for customer  $d={\lfloor mx \rfloor +1}$, and $t^{det}_x$ is the time at which adoption level reaches $x$ on following the price trajectory $p_{x'}, \forall x'\le x$ in the deterministic Bass model.
\end{restatable}
\begin{proof}
    \renewcommand{\Xdet}{X}
    Let $\Xdet_t$ be the  trajectory of adoption levels in the deterministic Bass model on 
    following price curve $p_x$. Recall that by \eqref{eq:detrministicBass},
    $m\frac{d\Xdet_t}{dt} = \lambda(p_t, \Xdet_t)$. Let $p_t \coloneqq p_{\Xdet_t}$ be
    the same price trajectory as $p_x$ but  indexed by time.
    
    First, note that for $m$ large enough we have $t^{det}_{n/m} \leq \frac{n}{e^{-p_{max}}\alpha (m-n)}\leq \frac{Te^C n}{\alpha(m-n)} \leq T$, This is because the last inequality holds for $m\ge n+ \frac{n}{\alpha} e^C$, i.e., for any $m$ satisfying  $m^{1/3} \geq \left( \frac{(\alpha+1)\alpha^{1/3}}{(\alpha+\beta)^{4/3}} \right) e^{C}$. Here we used the upper bound $p_{max} := \log(T) + C$ from  Assumption~\ref{assum:price upper bound}. 
    \removedshipra{If $m$ does not satisfy this condition, then since the expected pseudo regret is  lower bounded by $0$, and the first part in the right hand side of the lemma statement is trivially upper bounded by $\frac{n}{10}\le \frac{\alpha^{2}}{10(\alpha+\beta)^4} e^{2C} $, the lemma still holds \sacomment{how does the lemma still hold? I am not sure how you concluded that. if $m$ is arbitrarily small the second term goes to $0$. I think we should just modify the lemma statement to say $m^{1/3} \ge \frac{(\alpha+1)\alpha^{1/3}}{(\alpha+\beta)^{4/3}} e^C$. We anyway argue about assuming large enough $m$ in the proof of the lower bound theorem.}\sycomment{I was thinking that the this $\frac{n}{10}$ expression can hide in the big O notation of the second term, since it does not depend on $m$, and so is a "constant"}. } Therefore for the rest of the proof we can assume that $t^{det}_{n/m}\leq T$. 
    \begin{align}
        &\bbE[\preg(T)]\nonumber\\
        =&\bbE\left[\detV(0, T) - \sum\limits_{d=1}^{n}p_d - \sum\limits_{d=n+1}^{d_T} p_d \right]\nonumber\\ 
        =  &\bbE\left[\detV(0, T) - \sum\limits_{d=1}^{n}p_d\right] - \bbE\left[\bbE\left[\left.\sum\limits_{d=n+1}^{d_T} p_d\right| \cF_{n}\right]\right]\nonumber\\ 
        \ge&  \bbE\left[\detV(0, T) - \sum\limits_{d=1}^{n}p_d\right] - \bbE\left[\stochV(\frac{n}{m}, T-\tau_{n})\right] \nonumber\\ 
        =& \bbE\left[\detV(0, T) -  \sum\limits_{d=1}^{n}p_d - \detV(\frac{n}{m}, T-t^{det}_{n/m})\right] 
         + \bbE\left[\detV(\frac{n}{m}, T-t^{det}_{n/m}) - \stochV(n, T-\tau_{n})\right]\label{eq: preg as price error and diff between det stoch}
    \end{align}
    Given a particular sequence of $p_d$ for $d=1,\ldots, n$, and its' induced continuous version $p_x$ as defined in the lemma statement, the quantity inside the first expectation brackets is the cumulative ``disadvantage'' that the $p_x$ incurs on the deterministic Bass model, where disadvantage is defined in Lemma~\ref{lem: instantaneous impact of suboptimal price}. Therefore we can bound it as follows:
    \begin{align*}
        &\bbE\left[\detV(0, T) -  \sum\limits_{d=1}^{n}p_d - \detV(n/m, T-t^{det}_{n/m})\right]\\
        =& \bbE\left[\int_{0}^{t^{det}_{n/m}} \detV(\Xdet_t, T-t)-\detQ(\Xdet_t, T-t, p_t, dt)\right]\\
        =& \bbE\left[\int_{0}^{t^{det}_{n/m}} \frac{\detV(\Xdet_t, T-t)-p\lambda(p_t, \Xdet_t)dt - \detV(\Xdet_t+\lambda(p_t, \Xdet_t)dt/m, T-t-dt)   }{dt}\,dt \right]\\
        \geq& \bbE\left[\int_{0}^{t^{det}_{n/m}} \lambda(p_t, \Xdet_t)\min\left((\frac{1}{4}\popt_{\Xdet_t}-p_t)^2, \frac{1}{10}\right)dt \right]\\ 
        =&m\bbE\left[\int_{0}^{n/m} \min\left(\frac{1}{4}(\popt_x-p_x)^2, \frac{1}{10}\right) dx\right]
    \end{align*}
    where in the last step we applied change of variables 
    $\lambda(p_t, \Xdet_t)dt = m\,d\Xdet_t$

    The second part of \eqref{eq: preg as price error and diff between det stoch} can be bounded by using Lemm~\ref{lem: deterministic value better than stochastic value} and bounding the difference between $\tau_n$ and $t^{det}_{n/m}$:
    \begin{align*}
        &\bbE\left[\detV(\frac{n}{m}, T-t^{det}_{n/m}) - \stochV(n, T-\tau_{n})\right]\\ 
        =&\bbE\left[\detV(\frac{n}{m}, T-\tau_n) - \stochV(n, T-\tau_{n})\right] + \bbE\left[\detV(\frac{n}{m}, T-t^{det}_{n/m}) - \detV(n, T-\tau_{n})\right]\\ 
        \geq&\bbE\left[\detV(\frac{n}{m}, T-t^{det}_{n/m}) - \detV(n, T-\tau_{n})\right]
    \end{align*}
    Now recall from \eqref{eq: rate of change of deterministic value wrt time}
    and \eqref{eq: opt det price wrt value function} that $\popt(x, \alpha, \beta, T') =
    1-m\frac{\partial \detV(x, T')}{\partial x}$  and 
    \begin{align}
        \label{eq: rate of change of detV wrt T}
        \frac{\partial \detV(x, T')}{\partial T'} &= \popt\lambda(\popt, x) +
        \lambda(\popt,x) (1-\popt) = \lambda(\popt, x) \leq (\alpha+\beta)m
    \end{align}
    Using Lemma~\ref{lem: detarrival vs stocharrival}, which bounds the
    difference between $T-t^{det}_{n/m}$ and $T-\tau_{n}$, we have
    with probability $1-\delta$:
    \begin{align*}
    |\detV(\frac{n}{m}, T-t_{n/m}^{\text{det}}) - \detV(\frac{n}{m}, T-\tau_{n}) | 
    \leq&O\left(\frac{e^{p_{max}}(\alpha+\beta)^{1/3}}{\alpha^{1/3}}m^{1/3}\sqrt{\log(\frac{2}{\delta})}\right)
    \end{align*}
    Since $\detV(\frac{n}{m}, T-\tau_n)\leq \detV(\frac{n}{m},T)\leq m\pcurve_{max} = O(m\log((\alpha+\beta)T))$, we can set $\delta=\frac{1}{m}$ and 
    use Assumption~\ref{assum:price upper bound} to conclude that
    \begin{equation}
        \label{eq: main lower bound thm helper}
        \bbE\left[\left|\detV(\frac{n}{m}, T-t_{n/m}^{\text{det}}) - \detV(\frac{n}{m}, T-\tau_{n})\right|\right]\leq O\left(\newshipra{\frac{(\alpha+\beta)^{1/3}T e^C}{\alpha^{1/3}}}m^{1/3}\sqrt{\log(m)}\right)
    \end{equation}
    Here we also assumed that $m\geq \left(2\log(\frac{2}{\delta})\right)^{3/2}\left(\frac{\alpha+\beta}{\alpha}\right)^{2}$, which implies that $n \geq 2\log(\frac{2}{\delta})$. This satisfies the condition for Lemma~\ref{lem: detarrival vs stocharrival}. If this assumption on $m$ does not hold, then since the expected pseudo-regret is  lower bounded by zero, and the first term in the lemma statement is at most $\frac{n}{10}\leq \frac{1}{5}\log(\frac{2}{\delta})$, we have that  the lemma statement still holds for $\delta=\Theta(\frac{1}{m})$.
\end{proof}

\subsection{Step 2: missing lemmas and proofs}
\begin{restatable}{lemma}{lemdecisionrulelimit}
    For any $i$, let $[(p_1, I_1),\ldots (p_i, I_i)]$ be sequence of prices offered and inter-arrival times ($I_d:=\tau_{d}-\tau_{d-1}$) for the first $i$ customers, and let $\cF_i$ be the  filtration generated by these random variables. Here, prices could have been chosen
    adaptively using an arbitrary strategy as long as $p_i\in\cF_{i-1}$. Let
    $\pi^*_1 \neq \pi^*_2$ be any two fixed prices. Fix any deterministic
    algorithm that takes in the past $n$ observations as input and outputs a
    single price $\pi \in {\cF}_n$. Then, for any $\epsilon>0$ and $n$ such that $n\leq \left(\frac{\alpha
    m}{\epsilon }\right)^{2/3}$, at least one of the following holds:
    \begin{align*}
        \bbP_{\alpha, \beta}\left(
            \left|\pi- \pi^*_2)\right| \leq \left|\pi - \pi^*_1\right|
        \right) \geq \frac{1}{4}, & \text{ or,}\\
        \bbP_{\alpha, \beta+\epsilon}\left(
            \left|\pi- \pi^*_1\right| \leq \left|\pi - \pi^*_2\right| 
        \right) \geq \frac{1}{4}, &
    \end{align*}
    where $\bbP_{\alpha, \beta}$  denotes the probability distribution under the stochastic Bass model with parameters $\alpha, \beta$. Note that the
    only random quantity is $\pi$, which depends on the first $n$ observations.
    \label{lem: decision rule limit}
\end{restatable}
\begin{proof}
    Since $p_i\in \cF_{i-1}$, the probability of observing the sequence
    $[(p_1, I_1),\ldots (p_n, I_n)]$ is the product of the probabilities of
    observing $I_i$ given $\cF_{i-1}$. 
    \begin{align*}
        \bbP_{\alpha, \beta}\left([(p_1, I_1),\ldots (p_n, I_n)]\right) &= \prod\limits_{i=1}^n \bbP_{\alpha, \beta}\left(I_i | \cF_{i-1}\right)\\
        \bbP_{\alpha, \beta+\epsilon}\left([(p_1, I_1),\ldots (p_n, I_n)]\right) &= \prod\limits_{i=1}^n \bbP_{\alpha, \beta+\epsilon}\left(I_i | \cF_{i-1}\right)
    \end{align*}
    The $(\alpha, \beta)$ subscript denotes the fact that $I_i$'s are
    generated according to the stochastic Bass model with parameters $(\alpha, \beta)$.
    Since customer arrivals are Poisson, we know that given price $p_i$, the arrival time $I_i$ between customer $i-1$ and $i$ is exponentially distributed. Specifically, $I_i\sim
    \text{Exp}\left(\lambda_{\alpha, \beta}(p_i, \frac{i-1}{m})\right)$ in the
    $(\alpha, \beta)$ market, and
    $I_i\sim
    \text{Exp}\left(\lambda_{\alpha, \beta+\epsilon}(p_i, \frac{i-1}{m})\right)$
    in the $(\alpha, \beta+\epsilon)$ market, where $\lambda_{\alpha,
    \beta}(p, x) = e^{-p}(\alpha+\beta x)(1-x)m$. We can now bound the KL
    divergence between the joint distributions of the first $n$ observations
    between the two markets:
    \begin{align}
        \KL\left( \bbP_{\alpha, \beta+\epsilon}, \bbP_{\alpha, \beta}\right) 
        &= \KL\left(\prod\limits_{i=1}^n \bbP_{\alpha, \beta+\epsilon}\left(I_i | \cF_{i-1}\right), \prod\limits_{i=1}^n \bbP_{\alpha, \beta}\left(I_i | \cF_{i-1}\right)\right)\nonumber\\
        &= \sum\limits_{d=1}^{n}
        \KL\left(\left.
            \text{Exp}\left(\lambda_{\alpha, \beta+\epsilon}(p_d, \frac{d-1}{m})\right),
            \text{Exp}\left(\lambda_{\alpha, \beta}(p_d, \frac{d-1}{m})\right)
        \right| \right)\nonumber\\ 
        &\leq \sum\limits_{d=1}^n \frac{(\epsilon \frac{d-1}{m})^2}{2(\alpha+\beta \frac{d-1}{m})^2} \nonumber\\
        &\leq n \frac{(\epsilon \frac{n}{m})^2}{2\alpha^2} \label{eq:KL join exponential}
    \end{align}
    where the second equality follows from the standard decomposition of KL divergence (see for
    example Lemma 15.1 of \cite{lattimore2020bandit}) and the inequality
    follows from the following bound on the KL divergence of the two
    exponential distributions:

    The KL divergence for a general pair of exponentials is
    $\text{KL} (Exp (\lambda), Exp(\lambda_0)) =
    \log(\frac{\lambda_0}{\lambda}) + \frac{\lambda}{\lambda_0}-1$
    \begin{align*}
        &\KL \left( 
            \text{Exp}\left(\lambda_{\alpha, \beta+\epsilon}(p_d, \frac{d-1}{m})\right),
            \text{Exp}\left(\lambda_{\alpha, \beta}(p_d, \frac{d-1}{m})\right)
        \right)\nonumber\\
        =& \frac{\alpha+\beta \frac{d-1}{m}+ \epsilon\frac{d-1}{m}}{\alpha+\beta \frac{d-1}{m}}-1 + ln\left(\frac{\alpha+\beta \frac{d-1}{m}}{\alpha+\beta \frac{d-1}{m}+\epsilon\frac{d-1}{m}}\right)\nonumber\\
        =& \frac{\epsilon\frac{d-1}{m}}{\alpha+\beta \frac{d-1}{m}} 
          -\frac{\epsilon\frac{d-1}{m}}{\alpha+\beta \frac{d-1}{m}} 
          +\frac{(\epsilon \frac{d-1}{m})^2}{2(\alpha+\beta \frac{d-1}{m})^2} 
          -\frac{2(\epsilon \frac{d-1}{m})^3}{3!(\alpha+(\beta+\tilde\epsilon) \frac{d-1}{m})^3}\nonumber\\
        \leq& \frac{(\epsilon \frac{d-1}{m})^2}{2(\alpha+\beta \frac{d-1}{m})^2} 
    \end{align*}
    where in the last equality $\tilde \epsilon$ is some value in between $0$ and $\epsilon$.
    Now let $A_n$ be a sequence of $n$ observations such that the policy
    outputs a price that is closer to $\pi^*_1$. Using Pinsker's inequality
    and \eqref{eq:KL join exponential}, we can bound the difference in probability of observing this sequence in the two markets:
    \begin{align}
        2(\bbP_{\alpha, \beta}(A_n) -\bbP_{\alpha, \beta+\epsilon}(A_n))^2 &\leq \text{KL}(\bbP_{\alpha, \beta+\epsilon}, \bbP_{\alpha, \beta}) \nonumber \\
        |\bbP_{\alpha, \beta}(A_n) - \bbP_{\alpha, \beta+\epsilon}(A_n)| &\leq \sqrt{n\frac{(\epsilon \frac{n}{m})^2}{4\alpha^2}} 
        = \frac{\epsilon n^{3/2}}{2\alpha m} \label{probability KL lower bound percustomer}
    \end{align}
    Plugging in the upper bound on on $n$ from the lemma statement to \eqref{probability KL lower bound percustomer} gives us $|\bbP_{\alpha, \beta}(A_n) -
    \bbP_{\alpha, \beta+\epsilon}(A_n)| < \frac{1}{2}$.

    However, suppose \emph{neither} inequality in the lemma statement holds,
    then by the definition of $A_n$, we have that $|\bbP_{\alpha, \beta}(A_n)
    - \bbP_{\alpha, \beta+\epsilon}(A_n)| \geq |\frac{3}{4} - \frac{1}{4}| = \frac{1}{2}$ for $n^{3/2} \leq
    \frac{\alpha m}{\epsilon}$, which is a contradiction.
\end{proof}
Above lemma holds for any two prices $\pi^*_1 \ne \pi^*_2$. Readers should think of $\pi^*_1, \pi^*_2$ as the optimal prices  for customer $n+1$ in the $(\alpha, \beta)$ and $(\alpha,
\beta+\epsilon)$ market respectively.
To reduce clutter in the following Corollary statement, let $\popt_1 =
    \popt(x, \alpha, \beta, T'), \pi^*_2 = \popt(x, \alpha,
    \beta+\varepsilon, T')$, and let $\popt_M(x)$ be the optimal price for market $M$ i.e., $ \popt_M = \popt_1$ if $M=(\alpha, \beta)$ and $\popt_M=\popt_2$ otherwise.
\begin{corollary}
    Consider market $M$ sampled uniformly at random from $\{(\alpha,
    \beta), (\alpha, \beta+\epsilon)\}$, let $M'$ be the other market. 
    Suppose $n$ is such that $n^{3/2}\leq \frac{\alpha m}{\epsilon }$.
    Let $[(p_1,I_1),\ldots,(p_n,I_n)]$ be a sequence of $n$ observations
    generated from the market $M$, where $p_{i}\in\cF_{i-1}$.
    Fix any pricing algorithm that outputs $\pi$ based on the first $n$
    observations. 
    Then for any $x\in[0,1)$, any $T'$ such that $\pi^*_1\neq \pi^*_2$:
    \removedsy{$\pcurve(x, \alpha,
    \beta, T') \neq \pcurve(x, \alpha, \beta+\epsilon, T')$:
    $$
        \bbP\left(
            \left|\pi- \popt(x, M', T')\right| 
            \leq \left|\pi - \popt(x, M, T')\right| 
        \right) \geq \frac{1}{8}
    $$}
    $$
        \bbP\left(
            \left|\pi- \pi^*_{M'}(x)\right| 
            \leq \left|\pi - \pi^*_{M}(x)\right| 
        \right) \geq \frac{1}{8}
    $$
    where the probability is taken both with respect to the randomness from
    the stochastic arrival times, as well as the uniform sampling of
    the market parameters.
    \label{cor:decision rule limit}
\end{corollary}
\begin{proof}
    This directly follows from Lemma~\ref{lem: decision rule limit}.
    The extra factor of $1/2$ in the probability comes from the
    fact that we randomly picked one market out of two. 
\end{proof}

\subsection{Step 3: Lipschitz bound on the optimal pricing policy}
\begin{restatable}{lemma}{optimalpricedifferencedeterministic}
    \label{lem: optimal price difference deterministic}
    \removedsy{For any remaining time $T'\geq 0$, any $x\in[0,1)$, 
    $$\frac{\partial \pcurve(x, \alpha, \beta, T')}{\partial \beta} \leq
             \frac{-e\Xopt_{T'}(x)}{(\alpha+\beta \Xopt_{T'}(x))((\alpha+\beta)T'+e)}
                    +\frac{x}{\alpha+\beta x}$$
    
    In particular, for any $T'\geq\frac{(1+\sqrt{2})e}{\alpha+\beta}$ and $x\leq
    \frac{\alpha^2 e}{4(\alpha+\beta)^3 T'}$
    $$ \frac{\partial p^*(x, \alpha, \beta, T')}{\partial \beta} \leq \frac{-\alpha e}{28(\alpha+\beta)^2} $$}
    For any remaining time $T\geq\frac{(1+\sqrt{2})e}{\alpha+\beta}$ and $x\leq
    \frac{\alpha^2 e}{4(\alpha+\beta)^3 T}$
    $$ \frac{\partial \popt(x, \alpha, \beta, T)}{\partial \beta} \leq \frac{-\alpha e}{4(\alpha+\beta)^3T} $$
\end{restatable}
\begin{proof}
        Differentiating both sides of \eqref{eq: general XT_model_quadratic_sol1} with
        respect to $\beta$ gives us
        $$\frac{1}{1-\Xopt_T(x)} \frac{\partial \Xopt_T(x)}{\partial \beta} = \frac{T\Xopt_T(x)}{2\beta T\Xopt_T(x) + (\alpha-\beta)T +e}$$
        We omit the initial state argument $x$ and denote $\Xopt_T = \Xopt_T(x)$ in the following proof.
        \begin{align*}
            \frac{\partial \popt(x, \alpha, \beta, T)}{\partial \beta} 
            & = \frac{1}{1-\Xopt_T}\frac{\partial \Xopt_T}{\partial \beta}
                    -\frac{\Xopt_T}{\alpha+\beta \Xopt_T} 
                    -\frac{ \beta \frac{\partial \Xopt_T}{\partial \beta}}{\alpha+\beta \Xopt_T}
                    +\frac{x}{\alpha+\beta x}\\ 
            & = \frac{T\Xopt_T}{2\beta T\Xopt_T + (\alpha-\beta)T +e} - \frac{\Xopt_T}{\alpha+\beta \Xopt_T} 
                    -\frac{\beta}{\alpha+\beta \Xopt_T}\frac{(1-\Xopt_T)T\Xopt_T}{2\beta T\Xopt_T + (\alpha-\beta)T +e}
                    +\frac{x}{\alpha+\beta x}\\ 
            & = \frac{T\Xopt_T(\alpha+\beta \Xopt_T) - \Xopt_T(2\beta T\Xopt_T + (\alpha-\beta)T +e) - (1-\Xopt_T)\beta T\Xopt_T}{(\alpha+\beta \Xopt_T)(2\beta T\Xopt_T + (\alpha-\beta)T +e)}
                    +\frac{x}{\alpha+\beta x}\\ 
            & = \frac{-e\Xopt_T}{(\alpha+\beta \Xopt_T)(2\beta T\Xopt_T + (\alpha-\beta)T +e)}
                    +\frac{x}{\alpha+\beta x}\\ 
            & = \frac{-e\Xopt_T}{(\alpha+\beta \Xopt_T)\sqrt{(\alpha+\beta)^2T^2 + 2e(\alpha-\beta)T + e^2 + 4e\beta x T}}
                    +\frac{x}{\alpha+\beta x}\\ 
            &\leq \frac{-e\Xopt_T}{(\alpha+\beta \Xopt_T)((\alpha+\beta)T+e)}
                    +\frac{x}{\alpha+\beta x}
        \end{align*}
        We replaced $\Xopt_T$ with \eqref{eq: general XT_model_quadratic_sol2} in the last equality. 
        The last inequality follows from the fact that\\ $\sqrt{(\alpha+\beta)^2T^2 +
        2e(\alpha-\beta)T + e^2 + 4e\beta x T} \leq (\alpha+\beta)T+e$.
        In particular, if $T\geq\frac{(1+\sqrt{2})e}{\alpha+\beta}$, then \\
        $\sqrt{(\alpha+\beta)^2T^2 + 2e(\alpha-\beta)T + e^2 + 4e\beta x T}
        \leq \sqrt{2} (\alpha+\beta)T$. Also 
        $$\Xopt_T(x, \alpha, \beta) \geq\Xopt_T(0, \alpha, \beta)\geq 
        \Xopt_T(0, \alpha, 0)= \frac{\alpha T}{\alpha T+e} \geq
        \frac{1+\sqrt{2}}{2+\sqrt{2}}\frac{\alpha}{\alpha+\beta}$$
        The first inequality is easy to see from \eqref{eq: general XT_model_quadratic_sol2},  the second follows from Corollary~\ref{cor: beta improves final XT}, and the last inequality follows from the assumption on $T$. So the above can be simplified to 

        $$
            \frac{\partial \popt(x, \alpha, \beta, T)}{\partial \beta} 
            \leq \frac{-e\frac{1+\sqrt{2}}{2+\sqrt{2}}\frac{\alpha }{\alpha+\beta }}{(\alpha+\beta)^2T\sqrt{2}} + \frac{x}{\alpha+\beta x} = \frac{-\alpha e}{2(\alpha+\beta)^3T} + \frac{x}{\alpha+\beta x}
        $$
        Then, it easy to verify that for $x\leq \frac{\alpha^2 e}{4(\alpha+\beta)^3 T}$, 
        $
            \frac{\partial \popt(x, \alpha, \beta, T)}{\partial \beta} 
            \leq \frac{-\alpha e}{4(\alpha+\beta)^3T} 
        $
\end{proof}

\subsection{Step 4: Putting it all together for proof of Theorem~\ref{thm:mainlowerbound}}

We are now ready to prove our main lower bound result Theorem~\ref{thm:mainlowerbound}. In the following proof,
as defined earlier in Step 1, $p_x$ denotes the induced price trajectory from the algorithm's offered prices, and $t^{det}_x$ is the time when adoption level hits $x$ in the deterministic Bass model on following $p_x$ (both were defined in  detail in the paragraphs before Lemma~\ref{lem: detarrival vs stocharrival}).
\begin{proof}[Proof of Theorem \ref{thm:mainlowerbound}]
    Let $T=\frac{2(1+\sqrt{2})}{\alpha+\beta}$, $\epsilon=\frac{(\alpha+\beta)^2}{\alpha}$, $n=\left\lfloor \left(\frac{\alpha}{\alpha+\beta}\right)^{4/3} m^{2/3}\right\rfloor$.
    Randomly draw the Bass model parameters from $\{(\alpha, \beta), (\alpha, \beta+\epsilon)\}$ with equal probabilities. We denote the chosen market as $M$, and the other market $M'$. In the calculations that follow, we use $\bbE_M$ to indicate that the expectation is taken with respect to both the random choice of Bass model parameters as well as the randomness in the stochastic Bass model. 
    
    To reduce clutter, let $\popt_1(x) =
    \popt(x, \alpha, \beta, T-t^{det}_x), \popt_2(x) = \popt(x, \alpha,
    \beta+\varepsilon, T-t^{det}_x)$, and let $\pi^*_M(x)$ be the optimal price i.e., $ \pi^*_M(x) = \pi^*_1(x)$ if $M=(\alpha, \beta)$ and $\pi^*_M(x)=\pi^*_2(x)$ otherwise.

    \begin{align*}
    &\bbE_{M}[\preg(T)]\\ 
    \text{(Lemma~\ref{lem: lower bound pseudo regret price difference})}\geq &\bbE_{M}\left[m\int_{0}^{n/m}
    \min\left(\frac{1}{4}(\pi^*_M(x)-p_x)^2,
    \frac{1}{10}\right) dx\right]  -O\left(\frac{(\alpha+\beta)^{1/3}Te^C}{\alpha^{1/3}}m^{1/3}\sqrt{\log(m)}\right)\\ 
    \geq &  m\int_{0}^{n/m} \bbE_{M}\left[
        \min\left(\frac{1}{4}\left(\frac{\pi^*_1(x)-\pi^*_2(x)}{2} \right)^2,
        \frac{1}{10}\right)\1\left( \left|p_x- \pi^*_{M'}(x)\right| \leq
        \left|p_x - \pi^*_M(x)\right| \right)\right] dx \\ 
        &- O\left(\frac{m^{1/3}e^C}{(\alpha+\beta)^{2/3}\alpha^{1/3}}\sqrt{\log(m)}\right)\nonumber\\
        \text{(Lemma~\ref{lem: optimal price difference deterministic})}\geq
        & m\int_{0}^{n/m} \bbE_{M}\left[
        \min\left(\left(\frac{\alpha e\varepsilon}{16(\alpha+\beta)^3(T-t^{det}_x)}\right)^2,
        \frac{1}{10}\right)\1\left( \left|p_x- \pi^*_{M'}(x)\right| \leq
        \left|p_x - \pi^*_M(x)\right| \right)\right] dx\nonumber\\
        &- O\left(\frac{m^{1/3}e^C}{(\alpha+\beta)^{2/3}\alpha^{1/3}}\sqrt{\log(m)}\right)\nonumber\\
        \text{($T\geq T-t^{det}_x$)}\geq& m\int_{0}^{n/m} 
        \min\left(\left(\frac{e}{32(1+\sqrt{2})}\right)^2,
        \frac{1}{10}\right)\bbE_{M}\left[\1\left( \left|p_x- \pi^*_{M'}(x)\right| \leq
        \left|p_x - \pi^*_M(x)\right| \right)\right] dx\nonumber\\
        &- O\left(\frac{m^{1/3}e^C}{(\alpha+\beta)^{2/3}\alpha^{1/3}}\sqrt{\log(m)}\right)\nonumber\\
        \text{(Corollary~\ref{cor:decision rule limit})}\geq &
        \frac{n}{8} \min\left(\left(\frac{e}{32(1+\sqrt{2})}\right)^2, \frac{1}{10}\right)- O\left(\frac{m^{1/3}e^C}{(\alpha+\beta)^{2/3}\alpha^{1/3}}\sqrt{\log(m)}\right)\nonumber\\
        = &
        \Omega\left( \left(\frac{\alpha}{\alpha+\beta}\right)^{4/3}m^{2/3}\right) - O\left(\frac{m^{1/3}e^C}{(\alpha+\beta)^{2/3}\alpha^{1/3}}\sqrt{\log(m)}\right)\nonumber\\
        = &
        \Omega\left( \left(\frac{\alpha}{\alpha+\beta}\right)^{4/3}m^{2/3}\right)\nonumber
    \end{align*}
    
    Using the fact that the maximum
    expected regret between the two markets must be at least the average, we
    have that for at least one of $(\alpha, \beta)$ and $(\alpha,
    \beta+\epsilon)$,
    \begin{align*}
        \bbE[\preg(T)] \geq
        \Omega\left( \left(\frac{\alpha}{\alpha+\beta}\right)^{4/3}m^{2/3}\right)
        \nonumber
    \end{align*}
    
    Finally, we still need to check the assumptions of Lemma~\ref{lem: lower bound pseudo regret price difference}, Lemma~\ref{lem: optimal price difference deterministic} and Corollary~\ref{cor:decision rule limit} are satisfied for large enough $m$.
    
    To apply Lemma~\ref{lem: lower bound pseudo regret price difference}, we need $\Xopt_T\geq \left(\frac{\alpha}{\alpha+\beta}\right)^{4/3} m^{-1/3}=\frac{n}{m}$. Using the expanded notation $\Xopt(x, \alpha, \beta)$ introduced in Appendix~\ref{sec: det opt price properties},
    we know from Corollary~\ref{cor: beta improves final XT} that $\Xopt_T(0, \alpha, \beta) \geq \Xopt_T(0, \alpha, 0) = \frac{\alpha T}{\alpha T+e}$. It is easy to verify that for large enough $m$, specifically when $m^{1/3}\geq \left(\frac{\alpha}{\alpha+\beta}\right)^{4/3}\left(2+\frac{e(\alpha+\beta)}{\alpha(1+\sqrt{2})}\right)$ \newshipra{and $T=2(1+\sqrt{2})e/(\alpha+\beta)$}, $\frac{n}{m}\leq \frac{\alpha T}{\alpha T+e} \leq \Xopt_T(0,\alpha, \beta)=\Xopt_T$. \removedshipra{If this condition on $m$ does not hold, then the regret is at most $O(m p_{max}) = O(1)$, which does not change the asymptotic lower bound with respect to $m$.}
    
    To apply Lemma~\ref{lem: optimal price difference deterministic} we need the conditions $T'_x\coloneqq T-t^{det}_x \geq \frac{\newshipra{(1+\sqrt{2})e}}{\alpha+\beta}$ and $x\leq \frac{\alpha^2e}{4(\alpha+\beta)^3T'_x}$ for all $x\leq \frac{n}{m}$. 
    For any fixed $\alpha, \beta$, for large enough $m$, specifically for for $m^{1/3} \geq \frac{\alpha^{1/3}}{(\alpha+\beta)^{4/3}}(\alpha+e^C)$, where $C$ is a function of $\alpha,\beta$ as defined in Assumption~\ref{assum:price upper bound}, we have that
    $t^{det}_{n/m} \leq \frac{ n}{e^{-p_{max}}\alpha(m-n)} = \frac{e^{p_{max}}n}{\alpha(m-n)}\leq \frac{(1+\sqrt{2})e}{\alpha+\beta}$. The last inequality followed from simple algebraic calculations using the assumption on $m$ and Assumption~\ref{assum:price upper bound}. Therefore, $T-t^{det}_{x} \geq \frac{(1+\sqrt{2})e}{\alpha+\beta}$ for all $x\leq \frac{n}{m}$.  This satisfies the first condition. Furthermore, for $m^{1/3}\geq 8(1+\sqrt{2})\left(\frac{\alpha+\beta}{\alpha}\right)^{2/3}$, $x\leq n/m = \left(\frac{\alpha}{\alpha+\beta}\right)^{4/3} m^{-1/3} \leq  \frac{\alpha^2}{8(1+\sqrt{2})(\alpha+\beta)^2} = \frac{\alpha^2e}{4(\alpha+\beta)^3T} \leq \frac{\alpha^2e}{4(\alpha+\beta)^3T'_x}$. This satisfies the second condition. 
    
    Finally, to apply Corollary~\ref{cor:decision rule limit} we needed the condition $n^{3/2}\leq \frac{\alpha m}{\epsilon}$. By plugging in our choice of $m, \epsilon$, it is easy to verify that this condition is satisfied. 

\end{proof}

\section{Auxiliary Results}
\label{sec: auxiliary results}
\newcommand{\pub}{\bar p}
\begin{lemma}
\label{lem: const price ub and large T makes problem trivial}
Let $\pub$ be a constant upper bound on the prices that the seller can offer: $p_t\leq \pub$. Then there exists constants $C_0, C_1$ independent of $m$ such that for $T\geq C_0\log(m)+C_1$, there is a trivial policy that achieves $O\left(\log(\frac{1}{\delta})\right)$ regret.
\end{lemma}
\begin{proof}
Consider the trivial pricing policy where the seller sets the price to the highest possible value $\pub$ for the entire planning horizon. The rate of arrival of customer $d$ is given by $\lambda(\pub, \frac{d-1}{m}) = e^{-\pub} (\alpha+\beta\frac{d-1}{m})(m-d+1) \geq e^{-\pub}\alpha(m-d+1)$

We now divide the market of $m$ customers into a sequence of segments with geometrically decreasing lengths. Let $m_0 = 0, m_1 = \lfloor m/2\rfloor, m_2 = m_1 + \lfloor m/4\rfloor, \ldots, m_i = m_{i-1} + \lfloor m/2^{i}\rfloor$. Let $K=\lfloor \log_2(m) - \log_2\log(\frac{1}{\delta}) -2\rfloor$ such that $m_K = m-O(\log(\frac{1}{\delta}))$.
We call customers $d$ for $m_{i-1}< d \leq m_i$ the customers in segment $i$.

Since the price is constant, we use the following short hand notation $\lambda_d\coloneqq \lambda(\pub, \frac{d-1}{m})$. Note that by construction, $m-m_i\geq m/2^i$. So for $m_{i-1}< d\leq m_{i}$, $\lambda_d\geq \underline{\lambda}_i \coloneqq e^{-\pub}\alpha \frac{m}{2^{i}}$.

Let $I_d$ be the stochastic inter-arrival time between customer $d-1$ and $d$, then 
using Lemma~\ref{lem: general exponential arrival time martingale concentration bound}, we can obtain the following bound on the time it takes for all customers in segment $i$ to arrive:
\begin{align*}
    \bbP\left(\sum\limits_{d=m_{i-1}+1}^{m_i} I_d-\frac{1}{\lambda_d} \geq \epsilon_i\right) \leq exp\left(\frac{-\epsilon_i^2 \underline{\lambda}_i^2}{8 (m_i-m_{i-1})}\right)
\end{align*}
Setting this to $\delta$ and solving for $\epsilon_i$, one can verify that with probability $1-\delta$, $$\sum\limits_{d=m_{i-1}+1}^{m_i} I_d-\frac{1}{\lambda_d} \geq \epsilon_i \coloneqq \frac{4e^{\pub}}{\alpha} \sqrt{\log(\frac{1}{\delta}) \frac{2^{i}}{m}}$$
Note that we needed $\epsilon_i \leq \frac{2(m_i-m_{i-1})}{\underline{\lambda}_i}$ to apply Lemma~\ref{lem: general exponential arrival time martingale concentration bound}. This condition is satisfied for $i=1,\ldots, K$ because:
\begin{align*}
    \epsilon_i \leq \frac{2(m_i-m_{i-1})}{\underline{\lambda}_i} \iff & \sqrt{\log(\frac{1}{\delta}) \frac{2^i}{m}} \leq \frac{1}{2 } \iff i\leq 2\log_2\left(\frac{1}{2}\sqrt{\frac{m}{\log(1/\delta)}}\right)=\log(m)-\log_2\log(\frac{1}{\delta}) -2
\end{align*}
Applying a union bound to this result for all segments, we have that with probability at least $1-\delta \log_2(m)$, we have that 
\begin{align*}
    \sum\limits_{d=1}^{m_K} I_d \leq &\sum\limits_{d=1}^{m_K}\frac{1}{\lambda_d} + \sum_{i=1}^{K} \epsilon_i\\ 
    \leq& \sum_{i=1}^{K}\frac{m_i-m_{i-1}}{\underline{\lambda}_i} +  \sum_{i=1}^{K}\frac{4e^{\pub}}{\alpha} \sqrt{\log(\frac{1}{\delta}) \frac{2^{i}}{m}}\\ 
    = &2K\frac{e^{\pub}}{\alpha} + \frac{4e^{\pub}}{\alpha} \sqrt{\log(\frac{1}{\delta}) \frac{1}{m}}\frac{\sqrt{2}(2^{K/2}-1)}{\sqrt{2}-1}\\ 
    \leq & \frac{e^{\pub}}{\alpha}\left(2\log_2(m) + \frac{4\sqrt{2}}{\sqrt{2}-1}\sqrt{\log(\frac{1}{\delta})} \right)\\
    =& C_0 \log(m) + C_1
\end{align*}

This means that if $T\geq C_0\log(m) + C_1$, then with probability $1-\delta \log_2(m)$, the realized revenue is at least $m_K\pub$. Since the maximum possible revenue is $m\pub$, the regret is at most $\pub(m-m_K) = O(\log(\frac{1}{\delta}))$.
\end{proof}

\end{document}